\documentclass[runningheads]{llncs}
\usepackage[T1]{fontenc}
\usepackage{graphicx}
\usepackage{amsmath}
\usepackage{amssymb}
\usepackage{float}
\usepackage{enumerate}
\usepackage{tikz}
\usepackage{hyperref}

\usepackage{commandes}

\newcommand{\R}{\mathbb{R}}
\newcommand{\N}{\mathbb{N}}
\newcommand{\Z}{\mathbb{Z}}
\newcommand{\Sp}{\mathbb{S}}

\DeclareMathOperator{\Div}{div}
\DeclareMathOperator{\Spt}{spt}

\newcommand{\car}[1]{\raisebox{.5pt}{$\chi$}_{#1}}

\newcommand{\sca}[3][]{\left\langle {#2},{#3}\right\rangle_{#1}}
\begin{document}
\title{A variational method for curve extraction with
curvature-dependent energies}
%
%
\author{Majid Arthaud\inst{1,2,4}
\and
Antonin Chambolle\inst{3,2}\orcidID{0000-0002-9465-4659} \and
Vincent Duval\inst{2,3}\orcidID{0000-0002-7709-256X}}
\authorrunning{Majid Arthaud et al.}
%
\institute{ENPC, 6 Av.~Blaise Pascal, 77420 Champs-sur-Marne, France\\
\email{majid.arthaud@eleves.enpc.fr} \\[2mm] \and
INRIA Mokaplan, INRIA Paris, Paris-Dauphine, CNRS, France\\
\email{vincent.duval@inria.fr} \\[2mm] \and
CEREMADE, CNRS and Universit\'e Paris-Dauphine, PSL, Paris, France\\
\email{antonin.chambolle@ceremade.dauphine.fr} \and M.A. is now at University of Michigan, Ann Arbor, United States of America}
\maketitle              
\begin{abstract}
We introduce a variational approach for extracting curves
between a list of possible endpoints, based on
the discretization of an  energy and
Smirnov's decomposition theorem for vector fields.
It is  used to design a bi-level minimization approach to
automatically extract curves and 1D structures
from an image, which is mostly unsupervised.
We extend then the method to curvature-dependent energies,
using a now classical lifting of the curves in  the
space of positions and orientations equipped with an appropriate sub-Riemanian or Finslerian metric.

\keywords{Charges \and Numerical analysis  \and Bi-level problems \and Curve extraction \and Geodesic curves \and Active contours.}
\end{abstract}

\section{Introduction}

This paper is built upon the conference contribution~\cite{SSVMArthaudCD2025}, which is
extended towards more general
line energies, 3D examples and curvature-dependent energies.

We consider the problem of detecting (open) curves in images,
by a variant of active contours type models. {Although our final goal will be to simultaneously extract multiple curves and find their endpoints, for simplicity we first discuss the extraction of one curve given its endpoints. } Active
contours~\cite{snakes,CasellesGeometric,Cohen1996GlobalMF}
%
usually rely on the minimization of a potential along
a curve, with functionals of the form
\begin{equation}\label{eq:ActiveContour}
\int_{\Gamma} g(x,\tau)d\mathcal{H}^1
\end{equation}
which are minimal when the curve $\Gamma$ passes through
the lower values of the potential $g$ {defined from the observed image}  (here, $\mathcal{H}^1$
is the length measure, given by the $1$-dimensional Hausdorff measure). We consider a general setting, where
the potential is also allowed to depend on the tangent direction $\tau$ to the curve, yet, to start with,
the reader may discard this point. In general in these models, the curve
is assumed to be closed, but one could also specify the
endpoints, in order to find $\Gamma$ as a minimal
length curve (anisotropic geodesic) joining these two points~\cite{Cohen1996GlobalMF,DeschampsCohen}.
See for instance~\cite{Peyreetal} (in particular Section~3)
for a general overview
of these techniques.

A standard approach is then to solve an Eikonal equation (by fast-marching or fast-sweeping)
to evaluate the distance to one of the endpoints
(or several~\cite{CohenMultiple}),
and compute then the geodesic to some other endpoint(s). An advantage
is that the computation is fast, efficient, and can easily be adapted to many interesting frameworks (3D,
surfaces, nontrivial liftings such as in~\cite{LiYezzi,LiYezziCohen} for the tracking of the width of retina vessels, or as in~\cite{duits_optimal_2018} for their local orientation --- where here it is crucial to allow
for a spatial and orientation dependent weight $g$).
Alternatively, the problem of computing geodesics may be reformulated as a primal-dual problem as in~\cite{ennaji:hal-03620343}.

We consider here a (seemingly) different type of approach,
based on the representation of paths as measure fields
whose divergence is concentrated on their endpoints,
as proposed in~\cite{laville_smirnov,laville_algo}. {Consider for instance the problem:}
\begin{equation}\label{eq:ProblemContinuous}
\min_{z} \left\{ \int g\big(x,\tfrac{z}{|z|}\big)d|z|(x) : -\Div z = \delta_{B}-\delta_{A}\right\}
\end{equation}
where $z$ is a measure vector field of total variation measure $|z|$, and $A$ and $B$ are two given endpoints, while $z/|z|$ is the direction in the
polar decomposition
of the vectorial measure $z$. A typical example of such a measure  field is given by the integration along curves, that is, measures   of the form $\tau \mathcal{H}^1\resmes\Gamma$, where
$\tau$ is tangent to $\Gamma$ and $\Gamma$ is a curve. 
We show in Theorem~\ref{prop:smirnov} below (thanks to the celebrated
Smirnov theorem on extreme points of such vector fields~\cite{smirnov}) that the minimizers of~\eqref{eq:ProblemContinuous}
are superpositions of measures of this form, where each curve is geodesic between $A$ and $B$ for the cost $\int_\Gamma gd\mathcal{H}^1$. {One can then recover the curve(s) by integrating along the vector field $z$.}

Another path to this equivalence is through convex duality, which shows that~\eqref{eq:ProblemContinuous}
is essentially
equivalent to the standard approach described in~\cite{Peyreetal}, based on the Eikonal equation.
Yet a slight difference is that in~\eqref{eq:ProblemContinuous}, one can consider more general
constraints on $\Div z$, such as, for instance, consisting
of atomic measures with vanishing total mass, as we
propose in Section~\ref{sec:Iterative}.
This allows to compute simultaneously many geodesics between a possibly large family of points.
In that
case, the equivalence with the Eikonal equation approach
is broken and the setting we discuss becomes of interest.

After detailing our approach and showing how to derive
an algorithm for extracting curves from an image by
automatically detecting their endpoints, which was
essentially the contents of our contribution
in the SSVM~2025 conference proceedings~\cite{SSVMArthaudCD2025},
augmented by an extension in dimension 3,
we show how it is extended to  energies which
additionally encompass a convex function of the curvature, following the framework in~\cite{ChambollePockRoto} where the approach is developed
for a modified total variation regularizer.
A possible further extension could then be to
incorporate this setting in variational
formulations for more
general inverse problems (such as a
``deconvolution'' of thin structures) as in~\cite{laville_smirnov,laville_algo}.


The paper is organized as follows.
Before describing our numerical approach in Section~\ref{sec:Discrete} (which corresponds to
finding appropriate discretizations of~\eqref{eq:ProblemContinuous}), we
explain more precisely in the next section the connection between~\eqref{eq:ActiveContour} and \eqref{eq:ProblemContinuous}. Then, in Section~\ref{sec:Iterative}, we introduce a
bi-level method for automatically placing the endpoints. 
In Section~\ref{sec:curvature}, we explain how the proposed framework may be adapted to the ``roto-translational'' 
representation of oriented curves, in order to build weights $g$ which penalize the curvature.
Some proofs, and the precise description of the
optimization algorithm for solving the main convex
problem, are postponed to the Appendix.


\section{Charges and curves}\label{sec:charges}
To highlight the connection between the active contour functional~\eqref{eq:ActiveContour} and our model~\eqref{eq:ProblemContinuous}, we consider the space of ``normal charges'' $\mathcal{V}$  which was recently brought to the attention of the image processing community in~\cite{laville_smirnov}.
The idea consists in embedding finite curves in
a linear space of vector valued measures, similar
to the spaces of (normal) 1-currents~\cite{Federer} or 1-flat chains~\cite{Whitney}.

\subsection{The space of normal charges}
\label{sec:spacecharges}
 We work in a {compact arcwise connected} set 
$Q\subseteq \mathbb{R}^d$  (in practice $d$ will be $2$ or $3$, and $Q$ will be either a rectangle/cube or a cube periodic in one direction). The space of normal charges in $Q$ is the set of vector finite Radon measures supported in $Q$ and whose divergence is a finite Radon measure, that is,
\begin{equation}
    \mathcal{V} = \{z \in \mathcal{M}(\R^d)^d : \text{spt}(z)\subseteq Q, \, \text{div}(z) \in \mathcal{M}(\R^d)\},
\end{equation}
{where $\mathcal{M}(\R^d)$ is the space of finite (signed) Radon measures over $\R^d$.}
It is a normed space when equipped with the norm
\begin{equation} \label{continuous_norm}
\forall z\in \mathcal{V},\quad    \|z\|_\mathcal{V} = \|z\|_{\text{TV}} + \|\text{div}(z)\|_{\text{TV}},
\end{equation}
where $\|\cdot\|_{\text{TV}}$ denotes the total variation of (vector or scalar) Radon measures.

As normal charges are Radon measures, it is possible to define convex functionals of charges. We follow here the presentation of \cite[Sec. 2.6]{ambrosioFunctionsBoundedVariation2000}. Given a lower semi-continuous function $g\colon \R^d\times \R^d\rightarrow [0,+\infty]$, positively 1-homogeneous and convex in the second variable, we may define
\begin{align}\label{eq:defcvxmes}
    \forall z \in \mathcal{V},\quad G(z) = \int g\left(x, \frac{\diff z}{\diff \abs{z}}(x) \right) \diff \abs{z}(x).
\end{align}
Then, the functional $G\colon \mathcal{V}\rightarrow [0,+\infty]$ is convex positively homogeneous \cite[Prop. 2.37]{ambrosioFunctionsBoundedVariation2000} and Reshetnyak's lower semi-continuity theorem~\cite[Th. 2.38]{ambrosioFunctionsBoundedVariation2000} ensures that it is sequentially weak-* lower semi-continuous. We usually denote $G(z)$ by $\int g(x,z)$.

Furthermore, it is possible to define the measure $g(x,z)$ (see~\cite{bouchitte_integral_1988,demengel_convex_1984}) and to provide integral representations for the convex conjugate of $G$ (as in~\cite{bouchitte_integral_1988}).

\subsection{Charges induced by curves}
\label{sec:chargecurves}
An important example of a normal charge is induced by oriented curves of finite length, that is Lipschitz functions $\gamma \colon [a,b] \rightarrow Q$ for some $a,b \in  \R$.

The charge $\zgamma$ is then {the vectorial measure} defined by
\begin{equation}\label{eq:defzgamma}
\forall \varphi\in C_c(\R^d;\mathbb{R}^d),\quad    \langle \zgamma, \varphi \rangle = \int_a^b \varphi(\gamma(t))\cdot \gamma'(t) \diff t.
\end{equation}

One checks that $\mathrm{div}(\zgamma)= \delta_{\gamma(a)}-\delta_{\gamma(b)}$, so that  $\mathrm{div}(\zgamma)=0$ if and only if the curve is closed.

It is possible to give a geometric interpretation to~\eqref{eq:defzgamma}. Indeed, by the area formula and its consequence, the generalized change of variable~\cite[eq. 2.47]{ambrosioFunctionsBoundedVariation2000}, the following formula holds for every bounded Borel function $g\colon [a,b]\rightarrow \R$,
\begin{align*}
	\int_{\R^d}\left(\sum_{t\in [a,b]\cap \gamma^{(-1)}(y)}g(t)\right)\diff \Hh^1(y) = \int_a^b g(t)\abs{\gamma'(t)}\diff t.
\end{align*}
Setting $g(t)=1$ if $\gamma'(t)=0$ and $0$ otherwise, we see that for $\Hh^1$-a.e. $y \in  \R^d$, $\gamma'(t)\neq 0$ for every $t \in  \gamma^{(-1)}(y)$. Hence, given a test vector-field $\varphi \in C_c(\R^d;\R^d)$, we may set $g(t):= \varphi(\gamma(t))\cdot \frac{\gamma'(t)}{\abs{\gamma'(t)}}$ if $\abs{\gamma'(t)}\neq 0$, and $0$ otherwise, and we get
\begin{align*}
	\int_{\R^d}\varphi(y)\cdot\left(\sum_{t\in [a,b]\cap \gamma^{(-1)}(y)} \frac{\gamma'(t)}{\abs{\gamma'(t)}}\right)\diff \Hh^1(y) &= \int_a^b \varphi(\gamma(t))\cdot{\gamma'(t)}\diff t\\
	&= \langle  \zgamma,\varphi\rangle.
\end{align*}

If the curve is \emph{simple}, that is, if $\gamma$ is one-to-one on $[a,b)$, for $\Hh^1$-a.e. $y \in \mathrm{Im}\gamma$, there is a unique $t\in \gamma^{(-1)}(y)$ (and $\gamma$ is differentiable at $t$). We may thus define the  tangent vector as $\tau(y)=\frac{\gamma'(t)}{\abs{\gamma'(t)}}$, and the above equality states that the charge is equal to
\begin{align}
\zgamma =\tau \Hh^1\resmes (\mathrm{Im} \gamma) \quad &\mbox{and}\quad \abs{\zgamma}= \Hh^1\resmes (\mathrm{Im} \gamma).
\end{align}

Note that, as a consequence of the above discussion, \eqref{eq:defzgamma} is invariant by reparametrization, and, therefore, it is always possible  assume that $\abs{\gamma'(t)}\leq 1$ for a.e. $t$ and $\gamma$ is defined on $[0,L]$ for some $L>0$.

\subsection{Smirnov's decomposition theorem}
\label{sec:smirnovdecompo}

The landmark paper~\cite{smirnov} shows that any normal charge may be described as a superposition of charges induced by curves and their generalization. In order to state the main result we are interested in~\cite[Theorem C]{smirnov}, we briefly describe its setting.

Given some normal charge $z \in \mathcal{V}$, we say that $z$ \emph{decomposes into} $p\in \mathcal{V}$ and $q \in  \mathcal{V}$ if
\begin{align}
	z= p+ q \quad \mbox{and}\quad \abs{z} = \abs{p}+\abs{q}. \label{eq:defdecomposition}
\end{align}
The above equalities are understood in the sense of measures, and the expression $\abs{r}$ for $r \in \mathcal{V}$ denotes its variation measure, defined as $\abs{r}(E) = \sup \sum_{i} \abs{r(E_i)}$ for all Borel set $E \subseteq \R^d$, where the supremum is over all finite Borel subdivisions of $E$. In particular $\abs{r}(\R^d)= \|{r}\|_{\text{TV}}$.

Furthermore, we say that $z$ \emph{completely decomposes into $p$ and $q$} if~\eqref{eq:defdecomposition} holds and
\begin{align}
	\abs{\Div z} = \abs{\Div p}+\abs{\Div q}. \label{eq:defcompletedecompo}
\end{align}
Those definitions extend straightforwardly to finite (or integral) sums of charges.

In order to decompose some charge into curves of finite length, we endow the set of curves with the structure of a compact metric space. In view of Section~\ref{sec:chargecurves}, up to a reparametrization, a curve and the corresponding charge may be determined by some function $\gamma\colon \R \rightarrow \R^d$ which is $1$-Lipschitz. 
Smirnov defines  $\Lipd$  as the collection of all such functions together with the curve ``at infinity'' $f_\infty\colon t\mapsto \infty$. He endows it with a distance which metrizes uniform convergence on compact sets, and which makes it a compact space.
It is then possible to define a Borel measure on that space as in the following theorem.

\begin{theorem}[\protect{\cite[Thm. C]{smirnov}}]\label{thm:smirnov}
	Let $z\in \mathcal{V}$. Then there exist two normal charges $p,q \in \mathcal{V}$ such that $z$ completely decomposes into $p$ and $q$, $\Div p=0$, and $q$ completely decomposes into simple oriented curves of finite length. In other words, there exists some nonegative Borel measure $\sigma$ on $\Lipd$ such that
\begin{align}
	q &= \int \zgamma \diff \sigma(\gamma) \label{eq:smirnovz}\\
	\abs{q} &= \int \abs{\zgamma}\diff \sigma(\gamma),\label{eq:smirnovabs}\\
	\abs{\Div q} &= \int \abs{\Div \zgamma}\diff \sigma(\gamma)\label{eq:smirnovdiv}.
\end{align}
\end{theorem}
The ``conic combinations'' in (\ref{eq:smirnovz}-\ref{eq:smirnovdiv}) are in the weak-* sense, e.g.
\begin{align}
	\forall \varphi \in  C_c(\R^d),\quad  \langle q, \varphi\rangle &=\int \langle \zgamma, \varphi\rangle \diff \sigma(\gamma).
\end{align}

In turn, $p$ can also be completely decomposed as well, but the decomposition may involve elementary solenoids, that is charges induced by \emph{generalized curves} (see \cite[Def. and Thm. B]{smirnov}).
A generalized curve may consist in an infinitely winding curve of infinite length, but since we
show below that these objects do not appear in the solutions of our variational problem, we do not
describe them more precisely. See~\cite{smirnov} for more detail.


An alternative interpretation of Smirnov's results is that the extreme points of the unit ball of $\|\cdot\|_\mathcal{V}$ are the measures supported on simple oriented rectifiable curves, of norm one.
It has led Laville \textit{et al.} \cite{laville_smirnov,laville_algo} to use \eqref{continuous_norm} as a regularizer in inverse problems, as they prove that some solutions of their variational problem are superpositions of a finite number of measures induced by simple curves.



\subsection{A minimization problem with prescribed divergence}
Here, we focus on solving problems of the form
\begin{equation}\label{eq:ProblemContinuous2}
\min_{{z}} \left\{ \int_Q g\big(x,\tfrac{z}{|z|}\big)d|z|(x) : -\Div z = \mu\right\}
\end{equation}
where $\mu$ is a given measure with support in $Q$ and $\mu(Q)=\int_Q d\mu = 0$. This kind of problem is known in the field of branched transportation as Beckmann's problem (see for instance~\cite[Sec. 4.2]{santambrogio_optimal_2015} or~\cite{LOHMANN2022739}).
We make the following assumptions on the weight $g$.
\paragraph{Assumptions:}
\begin{enumerate}
    \item[(A0)] There exists $D>0$ such that for every $(x,y)\in Q^2$, there exists a Lipschitz curve $\gamma\colon [0,1]\rightarrow Q$ with $\gamma(0)=x$, $\gamma(1)=y$, and
    \begin{align}\label{eq:arcwiseD}
        \int_0^1 g(\gamma(t),\gamma'(t))\diff t \leq D.
    \end{align}
    \item[(A1)] The function $g\colon \R^d\times \R^d\rightarrow [0,+\infty]$ is lower semi-continuous, and it is convex positively 1-homogeneous in the second variable.
    \item[(A2)] There exists some constant $c>0$ such that $g(x,t)\geq c$ for all $(x,t)\in Q\times \SB^{d-1}(\R)$, where $\SB^{d-1}(\R)$ denotes the unit sphere of $\R^d$.
\end{enumerate}



The following result is a continuous version of known and standard results in graph theory and network flows (see also the discussion below~\eqref{eq:ProblemNaive}).
\begin{theorem}\label{prop:smirnov} Let $\mu\in\mathcal{M}(\R^d)$ with support in $Q$ and $\mu(Q)=0$, and assume that (A0), (A1) and (A2) hold. 

	Then, there exists a minimizer $z$ to~\eqref{eq:ProblemContinuous2}, and for any such minimizer there exists a Borel positive measure $\sigma$ defined on $\Lipd$, the set of oriented $1$-Lipschitz curves such that
\begin{align}
	z&\!=\! \int \zgamma \diff \sigma(\gamma), \label{eq:smirnovdecompoA}\\
	\abs{z} &\!=\! \int \abs{\zgamma} \diff\sigma(\gamma), \label{eq:smirnovdecompoB}\\
		\abs{\Div (z)} &\!=\! \int \abs{\Div(\zgamma)} \diff\sigma(\gamma), \label{eq:smirnovdecompoC}\\
    \int g(x,z) &= \int \left( \int g(x,\zgamma)\right)\diff \sigma(\gamma).\label{eq:smirnovdecompoD}
\end{align}
Moreover, $\sigma$-a.e. $\gamma$ is open, simple, and
\begin{enumerate}[(i)]
    \item\label{point1} $\gamma$ is a weighted geodesic curve in $Q$, \textit{i.e.}, minimizes  $\int g(\gamma(t),\gamma'(t)) dt$ for fixed beginning point $b(\gamma)$ and endpoint $e(\gamma)$,
    \item\label{point2} $b(\gamma)\in \Spt \mu^-$ and $e(\gamma)\in \Spt \mu^+$ where $\mu= \mu^+ - \mu^-$ is the Hahn-Jordan decomposition of $\mu=-\Div z$.
\end{enumerate}
\end{theorem}

The proof of Theorem~\ref{prop:smirnov} is given in Appendix~\ref{apx:decompo}.

In the particular case where $\mu$ is a sum of Dirac masses $\sum_i \pm\delta_{x_i}$, the solution
is carried by geodesic curves with endpoints on the $x_i$'s (oriented from the negative to the positive Diracs).

In \cite{laville_algo}, Laville~\textit{et al.}~exploit a
representer theorem which bounds the number of curves to solve their problem in an ``off-the-grid'' greedy approach relying on the Frank-Wolfe algorithm. The implementation is non-trivial, and requires a non-convex step. Since our setting possibly yields many more curves, we rather consider a convex approach based on a discretization. The difficulty arises in properly discretizing singular vector fields.
\begin{remark}
Contrary to what is written in~\cite{SSVMArthaudCD2025}, 
in general, the measure $\sigma$ is \emph{not} a probability measure.
\end{remark}
We now introduce a discretized version of these normal
charge, in order to use this representation in practical
image analysis to represent 1D curves in images. This
raises a few difficulties since in particular, a discrete
version of Theorem~\ref{prop:smirnov} does not exist
in general, except for very elementary anisotropic
curve energies.

\section{Discrete curves}\label{sec:Discrete}
In this section, we describe different discretization strategies and discuss their performance. 
To make the description easier to follow, we
first consider the bi-dimensional case (the generalization to 3D will be
straightforward) and
 weights which do not depend on the orientation. Orientation-dependent weights (Section~\ref{sec:curvature}) will require
some additional caution.


We consider 2D images of $N \times M$ pixels,
and we introduce $\mathcal{N}=\{(i,j):1\le i\le N, 1\le j\le M\}$ the set of nodes and $\mathcal{E}=\{(i+\frac{1}{2},j): 1\le i\le N-1,1\le j\le M\} \cup \{(i,j+\frac{1}{2}): 1\le i\le N,1\le j\le M-1\}$ the set of edges between the neighbouring nodes (where $(i+\frac{1}{2},j)$ denotes
the edge between the nodes $(i,j)$ and $(i+1,j)$, etc).

\subsection{Flux on a graph}
Playing the same role as the space of normal charges above, the space of discrete vector fields is defined as
\begin{equation}
    \mathcal{V}_d  = \mathbb{R}^{\mathcal{E}} \simeq (\mathbb{R})^{(N-1) \times M}
    \times (\mathbb{R})^{N \times (M-1)}
\end{equation}
For a vector field $v\in\mathcal{V}_d$, the
component $v_{i+\frac{1}{2},j}$ may be seen both
as a field between $(i,j)$ and $(i+1,j)$ (on the
edge), and a flux through the facet separating
the two pixels.

An important case of a vector field is the ``gradient'' of an image $u \in \mathbb{R}^\mathcal{N}\simeq\mathbb{R}^{N \times M}$, defined as the following finite difference operator:
\begin{equation}\label{eq:defDu}
    (D u)_{i+\frac{1}{2},j} = u_{i+1,j} - u_{i,j},
    \quad (Du)_{i,j+\frac{1}{2}} =  u_{i,j+1} - u_{i,j}.
\end{equation}
In this convention, the ``horizontal'' derivative
$(D u)_{i+\frac{1}{2},j}$
is defined on the edge between $(i,j)$ and $(i+1,j)$, for $i=1,\dots,N-1$, $j=1,\dots, M$
and similarly the ``vertical'' derivative is
defined on the edge between $(i,j)$ and $(i,j+1)$, for $i=1,\dots,N$ and $j=1,\dots,M-1$. 
Then, a discrete divergence operator is naturally
defined, for $v\in\mathcal{V}_d$, as $D^*v\in \R^\mathcal{N}$, given by $\sca[\R^\mathcal{N}]{D^*v}{u}=\sca[\mathcal{V}_d]{v}{Du}$ for all $u$ (with the scalar products given
by the canonical Euclidean products on the respective spaces).
With our convention in the definition of $D$ (we only
compute a difference when the two points are in the discrete
domain),
$D^*$ is a divergence with vanishing flux condition
on the boundary.

The advantage of this setting is its similarity with the continuous one (Sec.~\ref{sec:charges}). A straightforward adaptation of~\eqref{eq:ProblemContinuous2} is
\begin{equation}\label{eq:ProblemNaive}
    \min_{z \in \mathcal{V}_d, \, D^*z = \mu} \sum_{e\in \mathcal{E}} g_{e} |z_{e}|
\end{equation}
where $(g_e)_{e\in \mathcal{E}}$ is a collection of positive weights. It benefits from a discrete counterpart to Theorem~\ref{prop:smirnov}, which comes from well known  results in graph theory (see, e.g.~\cite[Prop.~3.10 and Rem.~3.12]{bonnans_gaubert}).

Nevertheless, as we illustrate below
(Fig.~\ref{fig:comma_l1l2}, center), this $\ell^1$-norm yields blurry results, due to the non-uniqueness of the
 corresponding geodesics. In addition, it measures the
``length'' of the loops or curves in a
very anisotropic way, only through their
horizontal and vertical projections.
In the active contour community, that phenomenon is well known, and people prefer to use discretizations of the isotropic Eikonal equation together with fast marching approaches over Dijkstra's  algorithm on the graph of the image~\cite{Cohen1996GlobalMF} in order to compute
geodesic.
\begin{figure}[htb]
    \centering
    \includegraphics[width=0.325\textwidth]{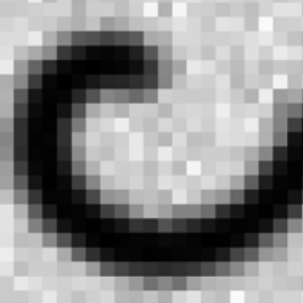}
      \includegraphics[width=0.325\textwidth]{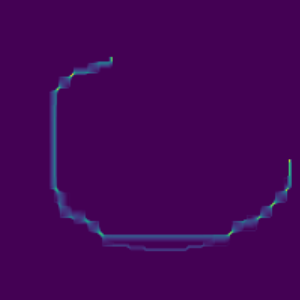}
    \includegraphics[width=0.325\textwidth]{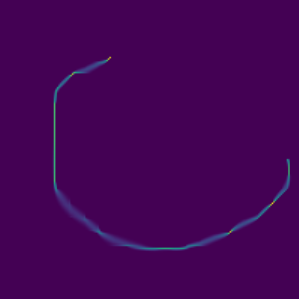}
    \caption{Left: $200 \times 200$ ``noisy comma'' shape
    (crop from~\cite[Fig.~1]{laville_algo}), middle: result with weighted $\ell^1$ norm, right: weighted $\ell^2$ norm.}
    \label{fig:comma_l1l2}
\end{figure}

One has therefore to consider other norms,
consistent (as $M,N\to\infty$) with the Euclidean
norm in~\eqref{eq:ProblemContinuous2}, even
if in doing so we lose at the discrete level
the decomposition theorem of a field as a
superimposition of curves.

\subsection{Isotropic discretizations}\label{isotropic_disc}
We introduce a family of discrete approximations of
the convex curve reconstruction problem~\eqref{eq:ProblemContinuous2}, with different anisotropy properties.
Given a weight function $g\in \mathbb{R}^{\mathcal{N}}$ with $0 \leq g_{i,j} \leq 1$ (derived from the gray level values of an image), and the scalar field of curve endpoints $\mu \in \mathbb{R}^{\mathcal{N}}$, the general form of
our problem is:
\begin{equation}\label{eq:ProblemDiscrete}
    \min_{z \in \mathcal{V}_d, \, D^*z = \mu} \sum_{i,j} g_{i,j} \|(Az)_{i,j}\|,
\end{equation}
where $A$ is an operator which sends $z$
to a $\R^2$-valued field in $\R^{\mathcal{N}}$
and $\|\cdot\|$ a norm in $\R^2$. For instance,
the weighted counterpart of the $\ell^1$-norm $\|z\|_1$
(as in~\eqref{eq:ProblemNaive}) corresponds to choosing
\begin{equation}\label{eq:A1}
(Az)_{i,j} = \begin{pmatrix} z_{i+\frac{1}{2},j} \\ z_{i,j+\frac{1}{2}} \end{pmatrix}, \quad \mbox{and}\ \forall x\in \R^2,\
\|x\| = |x_1|+|x_2|.
\end{equation}

By solving~\eqref{eq:ProblemDiscrete} we
 extract, through the discrete vector field $z$, dark curves (where $g$ is close to zero) on a light background (where $g$ is close to one). An experimental result with~\eqref{eq:A1} is shown in Fig.~\ref{fig:comma_l1l2}, middle: the weighted $\ell^1$ result is quite blurry, due to the fact
that several curves have almost the same energy, and that
the optimization outputs a convex combination
of these. In addition, horizontal
and vertical lines are strongly favored by the energy.


A variant yielding sharper and more isotropic results consists in replacing
the $\ell^1$-norm in~\eqref{eq:A1} with the Euclidean
norm $\|x\|=\sqrt{x_1^2+x_2^2}$. This corresponds
to a consistent discretization of~\eqref{eq:ProblemContinuous2} with forward
differences.
The result is shown in Fig.~\ref{fig:comma_l1l2}, right. The arbitrary choice of forward-forward differences for the gradient still induces some anisotropy at small scale. This is clearly seen when comparing the bottom left and the bottom right of the  comma, the first being more blurry than the second (see also Fig.~\ref{fig:commas_zoom}, top right).

A possibility to correct this anisotropic
behaviour and obtain a sharper result
is through averaging. As already
mentioned, in the spirit of discrete calculus \cite{grady}, the discrete vector field is living on the edges between the cells of the pixels, while the scalar fields of the grayscale levels or of the curve extremities should be seen as living on these cells. To re-center the norms of the $z_{i,j}$ on the cells, we replace in~\eqref{eq:A1} the operator
$A$ by a true averaging operator:
\[
(Az)_{i,j} = \frac{1}{2}\begin{pmatrix} z_{i-\frac{1}{2},j}+ z_{i+\frac{1}{2},j} \\ z_{i,j-\frac{1}{2}} +z_{i,j+\frac{1}{2}} \end{pmatrix},
\]
and using the $\ell^2$-norm again in~\eqref{eq:ProblemDiscrete}. This yields 
the discrete problem:
\begin{equation} \label{isotropic_2d}
    \min_{z \in \mathcal{V}_d, \, D^*z = \mu} \sum_{(i,j)\in\mathcal{N}} g_{i,j} \sqrt{\left(\frac{z_{i-\frac{1}{2},j} + z_{i+\frac{1}{2},j}}{2}\right)^2 +  \left(\frac{z_{i,j-\frac{1}{2}} + z_{i,j+\frac{1}{2}}}{2}\right)^2}.
\end{equation}
Notice that here, we have defined the coefficients outside of the image with zero padding (that is, $z_{\frac{1}{2},j} =z_{N+\frac{1}{2},j}= z_{i,\frac{1}{2}}=z_{i,M+\frac{1}{2}} = 0$),
which is consistent with the vanishing flux condition on the boundary.
\begin{figure}[htb]
    \centering
    \includegraphics[width=0.32\textwidth]{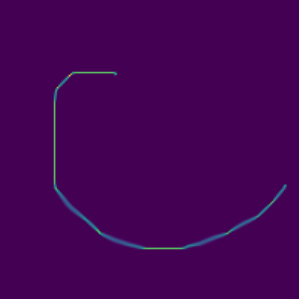}
    \begin{minipage}[b]{0.51\textwidth}
    \includegraphics[trim=25 20 10 90, clip, width=\textwidth]{virgule_B.pdf}
    \includegraphics[trim=25 20 10 90, clip, width=\textwidth]{virgule_C.pdf}
    \end{minipage}
    \caption{Left: result with the weighted and averaged $\ell^2$-norm  on the same ``noisy comma'' image. Right: zoom of the same result (bottom) and zoom without averaging (top).}
    \label{fig:commas_zoom}
\end{figure}
The results are very isotropic and sharper than those obtained through the weighted $\ell^2$ norm, as seen in Fig.~\ref{fig:commas_zoom}, left. A zoom on the bottom part  compares these two versions (Fig.~\ref{fig:commas_zoom}, right).
\begin{figure}[htb]
\centering
    \includegraphics[angle=90,width=.49\textwidth,clip, trim= 40 0 50 0]{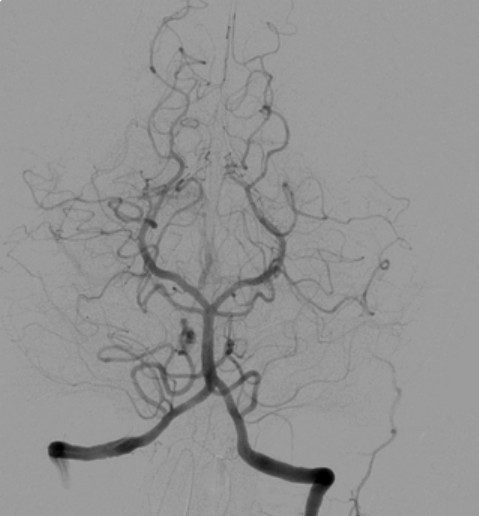}
    \includegraphics[angle=90,width=.49\textwidth,clip, trim= 40 0 50 0]{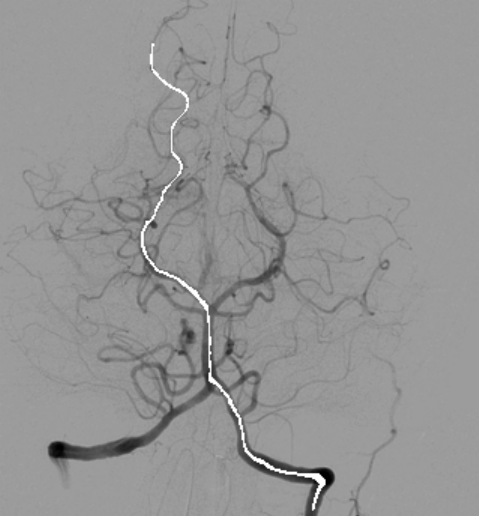}
    \caption{Left: an angiogram (source: Wikipedia), right: geodesic computed by our method (compare
    for instance with~\cite[Fig.~4(d)]{DeschampsCohen}).}
    \label{fig:resultat_angio}
\end{figure}
In Fig.~\ref{fig:resultat_angio}, we illustrate the use of this method on an angiogram. The result is a 
geodesic similar to the results obtained in the literature on minimal paths~\cite{DeschampsCohen}.


\subsection{Optimization}
All of the experimental results exposed in this document have been implemented with a primal-dual algorithm \cite{chambolle_pock}, or ``PDHG'', {accelerated
by a reprojection on the constraint
$D^*z=\mu$ computed with FFTW3~\cite{FFTW3}}, 
which solves the Lagrangian saddle point problem:
\begin{equation}\label{eq:saddlePD}
\min_{z \in \mathcal{V}_d} \max_{p \in (\mathbb{R}^\mathcal{N})^2} \langle A z, p \rangle + \car{\{z\,:\,D^*z = \mu\}} - \car{\{p\,:\, \|p_{i,j}\|_* \leq g_{i,j}\, \forall i, j, \}}.
\end{equation}
Here, $\|\cdot\|_*$ is the dual norm of the
norm $\|\cdot\|$ in~\eqref{eq:ProblemDiscrete}
(defined by $\|q\|_* = \sup_{\|p\|\le 1} q\cdot p$),
and $\chi$ denotes a characteristic function
in the classical sense of convex analysis
($0$ if the condition is satisfied, $+\infty$ else).
The PDHG algorithm
described in
Appendix~\ref{sec:PD} is elementary to implement
and only requires the matrix-vector products with
$A$ and its adjoint $A^*$, and
the knowledge of the projections onto the sets whose
characteristic function appear in~\eqref{eq:saddlePD}.
The main advantage of this approach is its versatility: it is straightforward to adapt to any
linear operator $A$, or
to change the penalization of $(Az)_{i,j}$ by
suitably modifying the constraint set
for $p$ (here, $\|p_{i,j}\|_* \leq g_{i,j}$;
more complex convex, one-homogeneous penalizations
will be considered in Section~\ref{sec:curvature}). See the
full description in Appendix~\ref{sec:PD}.

The idea of using an averaging operator comes
from the context of discretizing the total
variation, and was introduced in~\cite{condat}
and further developed in~\cite{chambolle_learning}.
An interesting topic for research would be to analyze and
reproduce more closely their
findings, or develop variants in the context of curve extraction.
\begin{figure}[bht!]
    \centering
    \includegraphics[width=0.3\textwidth]{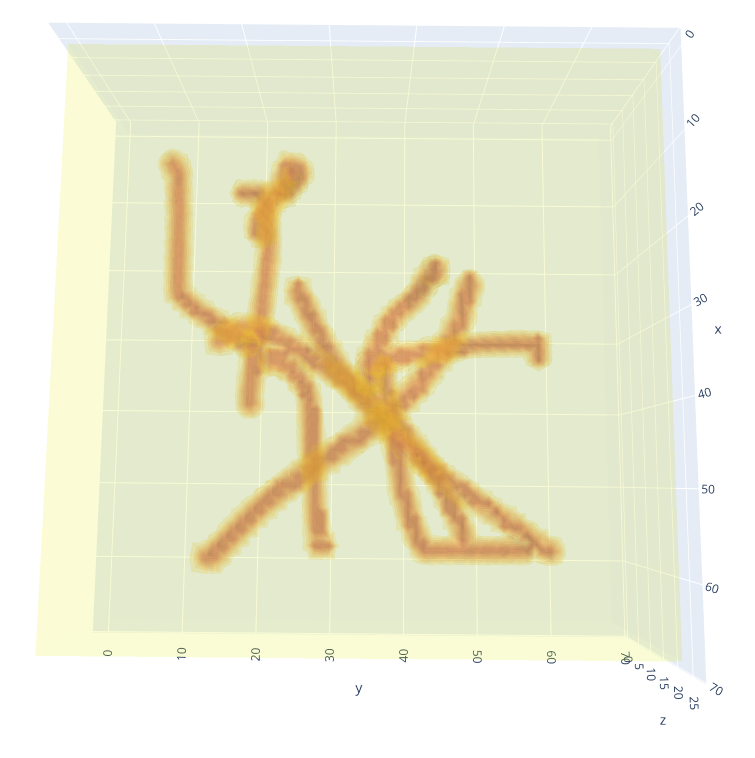}
    \includegraphics[width=0.3\textwidth]{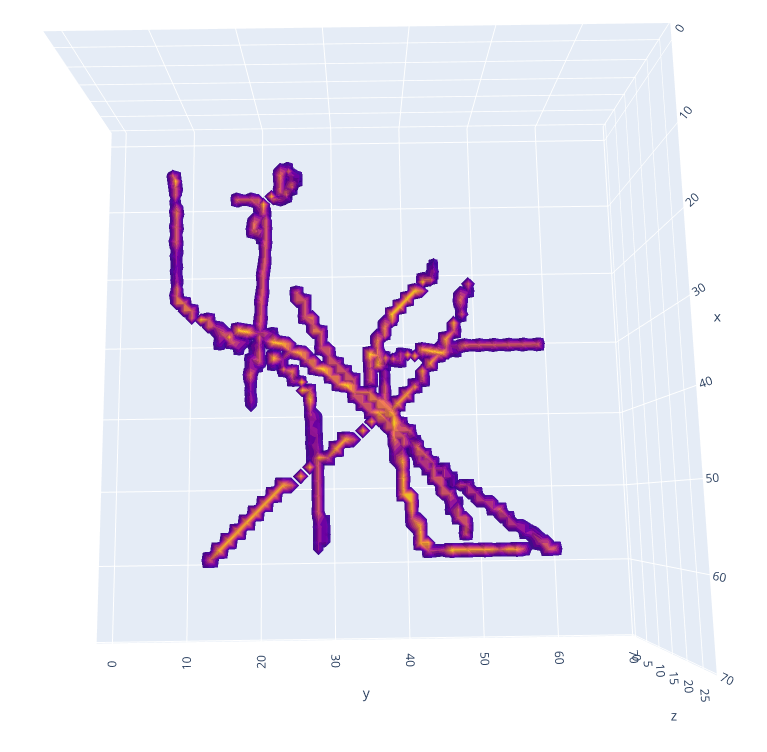}
    \includegraphics[width=0.4\textwidth]{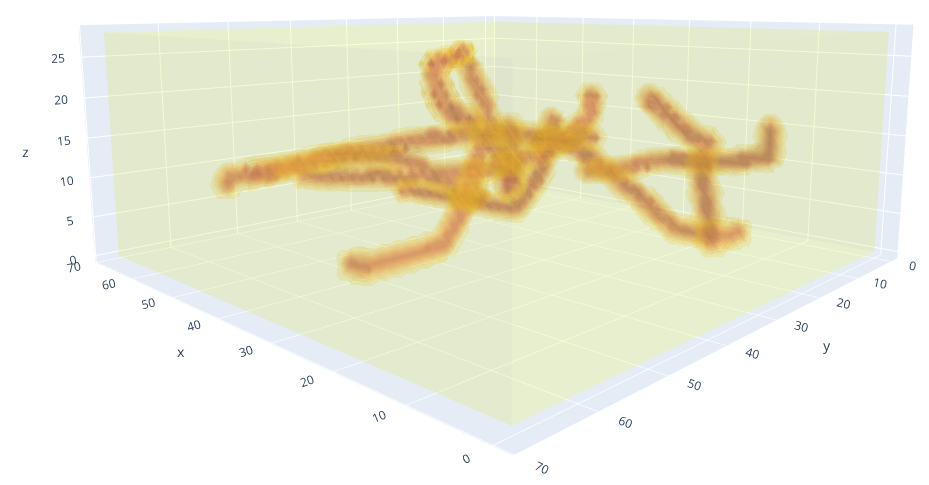}
    \includegraphics[width=0.4\textwidth]{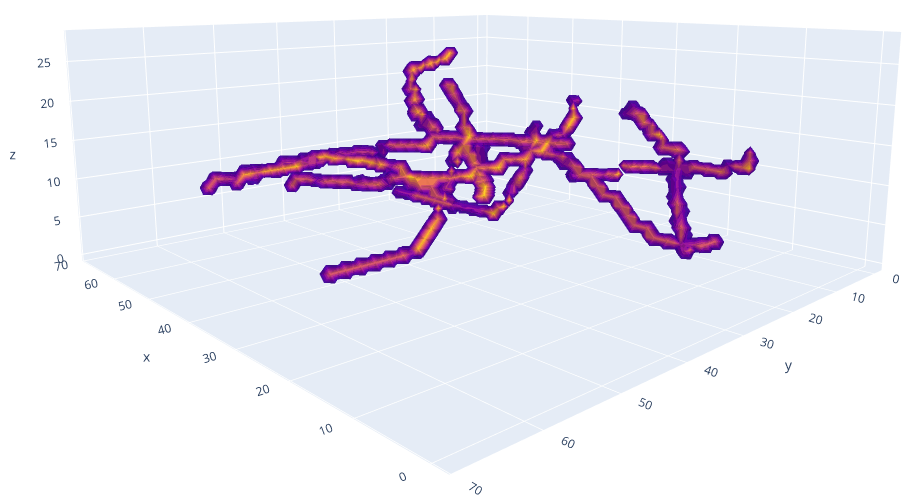}
    \caption{Result on a 3D potential: the left images correspond to the $70 \times 70 \times 30$ potential, and the right images are the retrieved curves. Now the
    result is obtained after 2000 primal-dual steps, with fixed endpoints, for a wall time of about 6 min.}
    \label{fig:resultat_3D}
\end{figure}

The extension to a 3D volume of this method is
straightforward, with the same formulas, the same averaging
and the same algorithm.
We show an example (on synthetic data
and as before with given endpoints) in
Figure~\ref{fig:resultat_3D}.

\section{An iterative discrete curve reconstruction algorithm} \label{sec:Iterative}
\subsection{Finding the endpoints}\label{sec:EndPoints}
The discrete convex curve reconstruction problem of the previous section assumes that the endpoints (represented by discrete Dirac masses) of the curves to reconstruct were given. Now, we propose an iterative algorithm which attempts to retrieve discrete curves in images with no prior knowledge of their positions. The only parameter are the
maximal number of endpoints and a threshold $g_{max}$, chosen so that the
curves should roughly describe the set $\{g\le g_{max}\}$.
Formally, given $g$ and $g_{max}$, and a maximum number of source/sink pairs  $n_0$,
we aim at solving the non-convex bi-level optimization problem (here $Q$ is a rectangle representing the domain of the image):
\begin{equation} \label{Non convex prob}
\begin{aligned}
\min_{\begin{subarray}{c}1\le n\le n_0 \\ (s_i,t_i)_{i=1}^n \in Q^{2n}\end{subarray}} &\left\{ \int_Q (g\left(x,\frac{\diff z}{\diff \abs{z}}(s)\right)-g_{max})d|z|(x):\right. \\
& \left. \qquad\qquad z\text{ solves~\eqref{eq:ProblemContinuous2} for }
\mu=\sum_{i=1}^n \delta_{t_i} - \sum_{i=1}^n \delta_{s_i}\right\}.
\end{aligned}
\end{equation}
The idea is to try to find the endpoints of the curves present in the image, by maximizing the length of the minimal curves between these endpoints inside the sublevel set $\{g \leq g_{max}\}$.

To implement this minimization, we need to describe the gradient with respect to
a source/sink point $s_i$ or $t_i$ of the energy
in~\eqref{Non convex prob}. Since the optimal
measure is expected to be carried by  geodesic curves joining each $s_i$ to some point $t_j$,
the largest change for the energy corresponds
to moving $s_i$, $t_j$ \textit{along} the geodesic
(that is, in direction $\pm z(s_i)$ or $\pm z(t_i)$), towards the direction which \textit{shortens} the curve if $g(x)>g_{max}$, and
the direction which \textit{lenghtens} the curve
if $g(x)<g_{max}$. At equilibrium, we also
check whether moving the sources/sinks in the direction
of $\nabla g$ improves the criterion.
Also, we sometimes need to decrease the number of sources/sinks.

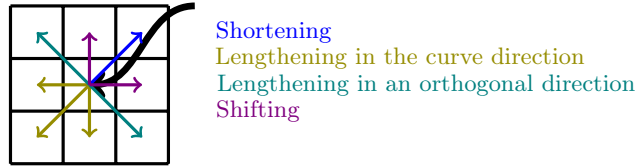
\begin{figure}[htb]
\centering
\begin{tikzpicture}[x=0.7cm, y=0.7cm, step=0.7cm]

\draw[very thick] (0,0) grid (3,3);
\draw[line width=.9mm, ->] (3.5,3) to[out=180,in=0] (1.5,1.5);

\draw[blue, very thick, ->] (1.5,1.5)--(2.5,2.5);
\draw[olive, very thick, ->] (1.5,1.5)--(.5,.5);
\draw[olive, very thick, ->] (1.5,1.5)--(1.5,.5);
\draw[olive, very thick, ->] (1.5,1.5)--(.5,1.5);
\draw[teal, very thick, ->] (1.5,1.5)--(.5,2.5);
\draw[teal, very thick, ->] (1.5,1.5)--(2.5,.5);
\draw[violet, very thick, ->] (1.5,1.5)--(1.5,2.5);
\draw[violet, very thick, ->] (1.5,1.5)--(2.5,1.5);

\node[blue] at (5,2.5) {Shortening};
\node[olive] at (7.4,2) {Lengthening in the curve direction};
\node[teal] at (7.9,1.5) {Lengthening in an orthogonal direction};
\node[violet] at (4.7,1) {Shifting};

\end{tikzpicture}
\caption{For the discretization of the continuous curve drawn in black and the Dirac mass at its endpoint, we represent the four different cases of the algorithm.}
\label{fig:directions}
\end{figure}\vspace{-4mm}
We now describe the discrete implementation in two dimensions.
In what follows,
for $x\in\mathcal{N}$ we denote by $\delta^x\in\R^\mathcal{N}$ a discrete Dirac mass at $x$, given by $\delta^x_{i,j}=1$ if $x=(i,j)$ and
$\delta^x_{i,j}=0$ else.
The algorithm is the following:
\begin{enumerate}
    \item Choose a set of initial points $\mathcal{S}\subset \mathcal{N}$ and a set
    of final points $\mathcal{T}\subset \mathcal{N}$
    of same cardinality $n_0$, and initialize $\mu=\sum_{s\in \mathcal{S}}\delta^{s} - \sum_{t\in \mathcal{T}} \delta^{t}\in \R^\mathcal{N}$ as a sum of discrete Dirac masses with total sum zero;
    \item Solve \eqref{eq:ProblemDiscrete} (e.g. with the primal-dual algorithm~\cite{chambolle_pock}, see Appendix~\ref{sec:PD});
    \item For each Dirac mass at some $x=(i,j)\in \mathcal{N}$: estimate the discrete curve orientation near $x$, by the average of the discrete vector field $z$ in the $3 \times 3$ square around $x$. 
    We distinguish four cases:
            \begin{enumerate}
                \item The \textbf{shortening} case: if $g > g_{max}$ at the current pixel of the Dirac mass, move the Dirac mass to the pixel in the $3 \times 3$ square in the curve shortening direction;
                \item The \textbf{lengthening in the curve direction} case: if the condition of the previous case is not satisfied, and $g \leq g_{max}$ on one of the pixels opposite or at 45° of the opposite of the curve shortening direction, and this pixel has not yet been visited by this Dirac mass, move the Dirac mass to the pixel verifying this condition with the lowest potential $g$;
                \item The \textbf{lengthening in an orthogonal direction} case: if the conditions of the two previous cases are not satisfied, and $g \leq g_{max}$ on one of the pixels in one of the two orthogonal directions to the  shortening direction, and this pixel has not yet been visited by this Dirac mass, move the Dirac mass to the pixel with this condition and the lowest potential $g$;
                \item The \textbf{shifting} case: if the conditions of the previous three cases are not satisfied, and $g \leq g_{max}$ on one of the pixels in one of the two directions at 45° of the curve shortening direction, and this pixel has not yet been visited by this Dirac mass, move the Dirac mass to the pixel verifying this condition with the lowest potential $g$;
            \end{enumerate}
    \item Reiterate steps 2 and 3 until all of the Dirac masses have converged.
\end{enumerate}
The directions in the $3 \times 3$ square around a Dirac mass are represented in Fig.~\ref{fig:directions}. The Dirac masses converge in practice, as they cannot go from the ``lengthening and shifting'' stage to the ``shortening'' stage, and in the ``lengthening and shifting'' stage, each pixel can only be visited once. In some instances (see next Section), we
may have to merge some endpoints/curves and start
again the algorithm to improve the results, ending in
a number of endpoints which is always less than the
initial choice.
It remains unclear how to extend properly
these rules in 3D and to cope with the fact
that 3D curves are more likely to avoid the
zones of interest, which makes trickier the
initialization to reproduce our 2D results.
This is a topic for future study and
experiments.

\begin{figure}[htb]
    \centering
    \includegraphics[width=.28\textwidth]{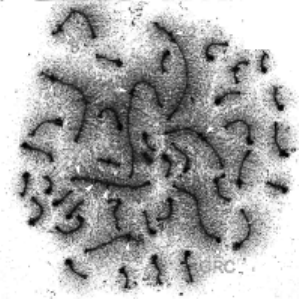}
    \includegraphics[width=.28\textwidth]{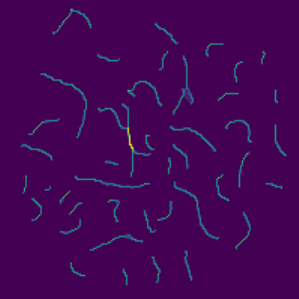}
    \caption{Result of the iterative algorithm after post-processing, on a $200 \times 200$ image of 42 bird chromosomes. 
     {Obtained after 100 iterations of the iterative algorithm, each consisting in 60 primal-dual steps, for a total of 1~min.~05~s.\ wall time on an average laptop. A post-processing, which consists in 5000 primal-dual steps on the last set of Dirac masses, adds another 1~min.~12~s.~wall time.}}
    \label{fig:resultat_chromosomes}
\end{figure}

\begin{figure}[htb]
    \centering
        \begin{tikzpicture}
    \node(a) at (0,0) {\includegraphics[width=0.235\textwidth]{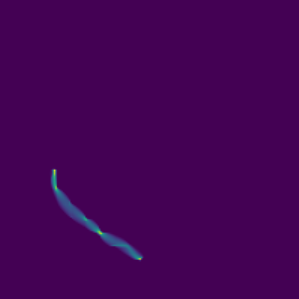}};
    \node[white] at (a.south west)
    [anchor=center,
    xshift=7mm,
    yshift=3mm
    ]
    {$k=10$};
    \node(b) at (0.25\textwidth,0) {\includegraphics[width=0.235\textwidth]{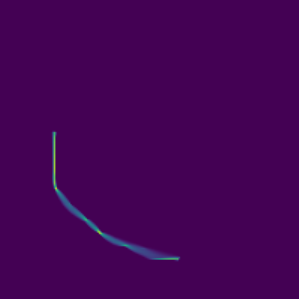}};
    \node[white] at (b.south west)
    [anchor=center,
    xshift=7mm,
    yshift=3mm
    ]
    {$k=60$};
    \node(c) at (0.50\textwidth,0) {\includegraphics[width=0.235\textwidth]{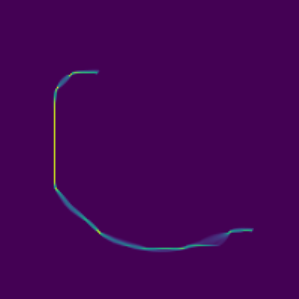}};
    \node[white] at (c.south west)
    [anchor=center,
    xshift=7mm,
    yshift=3mm
    ]
    {$k=190$};
    \node(d) at (0.75\textwidth,0) {\includegraphics[width=0.235\textwidth]{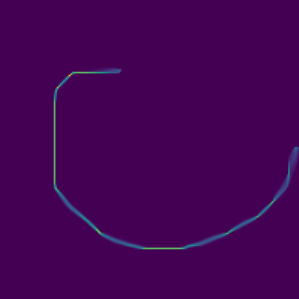}};
    \node[white] at (d.south west)
    [anchor=center,
    xshift=7mm,
    yshift=3mm
    ]
    {$k=380$};
    \node at (a.north east)
    [
    anchor=center,
    xshift=-8.6mm,
    yshift=-8.6mm
    ]
    {
        \includegraphics[width=0.12\textwidth]{virgule.pdf}
    };
    \end{tikzpicture}
            \begin{tikzpicture}
    \node(a) at (0,0) {\includegraphics[width=0.235\textwidth]{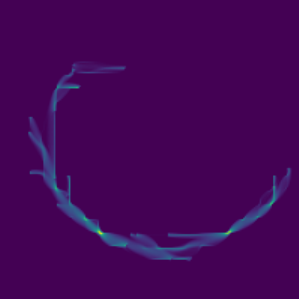}};
    \node[white] at (a.south west)
    [anchor=center,
    xshift=7mm,
    yshift=3mm
    ]
    {$k=30$};
    \node(b) at (0.25\textwidth,0) {\includegraphics[width=0.235\textwidth]{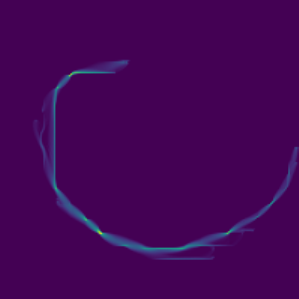}};
    \node[white] at (b.south west)
    [anchor=center,
    xshift=7mm,
    yshift=3mm
    ]
    {$k=160$};
    \node(c) at (0.50\textwidth,0) {\includegraphics[width=0.235\textwidth]{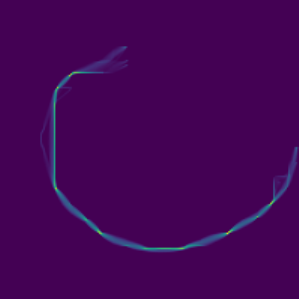}};
    \node[white] at (c.south west)
    [anchor=center,
    xshift=7mm,
    yshift=3mm
    ]
    {$k=390$};
    \node(d) at (0.75\textwidth,0) {\includegraphics[width=0.235\textwidth]{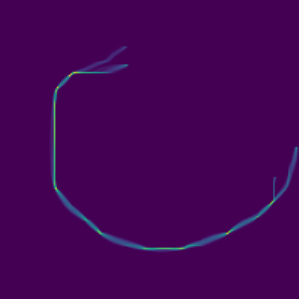}};
    \node[white] at (d.south west)
    [anchor=center,
    xshift=13mm,
    yshift=3mm
    ]
    {post-processing};
    \node at (a.north east)
    [
    anchor=center,
    xshift=-8.6mm,
    yshift=-8.6mm
    ]
    {
        \includegraphics[width=0.12\textwidth]{virgule.pdf}
    };
    \end{tikzpicture}
            \begin{tikzpicture}
    \node(a) at (0,0) {\includegraphics[width=0.235\textwidth]{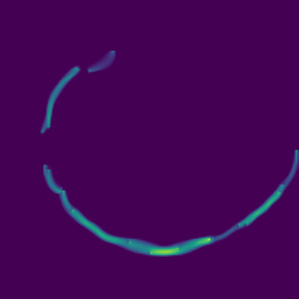}};
    \node[white] at (a.south west)
    [anchor=center,
    xshift=7mm,
    yshift=3mm
    ]
    {$k=30$};
    \node(b) at (0.25\textwidth,0) {\includegraphics[width=0.235\textwidth]{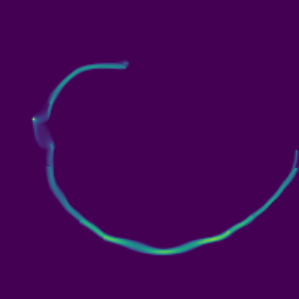}};
    \node[white] at (b.south west)
    [anchor=center,
    xshift=7mm,
    yshift=3mm
    ]
    {$k=110$};
    \node(c) at (0.50\textwidth,0) {\includegraphics[width=0.235\textwidth]{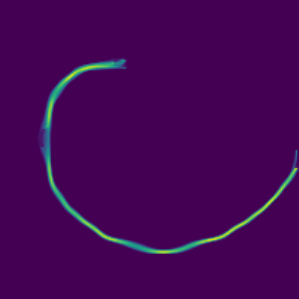}};
    \node[white] at (c.south west)
    [anchor=center,
    xshift=7mm,
    yshift=3mm
    ]
    {$k=280$};
    \node(d) at (0.75\textwidth,0) {\includegraphics[width=0.235\textwidth]{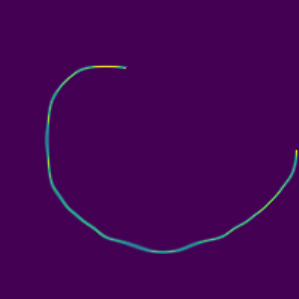}};
    \node[white] at (d.south west)
    [anchor=center,
    xshift=13mm,
    yshift=3mm
    ]
    {post-processing};
    \node at (a.north east)
    [
    anchor=center,
    xshift=-8.6mm,
    yshift=-8.6mm
    ]
    {
        \includegraphics[width=0.12\textwidth]{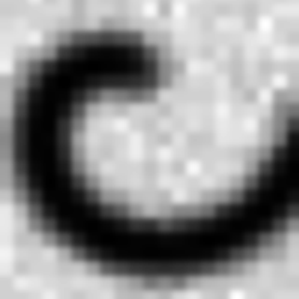}
    };
    \end{tikzpicture}
        \caption{Result of the iterative algorithm on the $200 \times 200$ ``comma'' shape. The iteration number $k$ for the algorithm is indicated.
        Top: $\mu$ initialized as a pair of Dirac masses of opposite intensities randomly chosen in the sublevel set $\{ g \leq g_{max} \}$.
       Middle: $\mu$ initialized as $15$ pairs of Dirac masses of opposite intensities randomly chosen in the sublevel set $\{ g \leq g_{max} \}$. 
        Bottom: same as ``middle'', but the ``comma'' shape is blurred with a Gaussian filter of standard deviation 2. 
        {Each iteration had 80 primal-dual steps.
        The middle experiment was the longest, taking 5~min.~07~s.~wall time after 390 iterations on an average laptop. A post-processing of 5000 primal-dual steps on the last set of Dirac masses added 1~min.~03~s.~wall time.}
        }
        \label{fig:grille_resultats}
\end{figure}
\subsection{Numerical results}\label{sec:Num}
We show a few results obtained with
our implementation\footnote{available
at \url{https://github.com/majidarth/Discrete-curve-reconstruction}.}.
The computation of the optimal $z$ given
$\mu$ relies on a primal-dual method~\cite{chambolle_pock}, see Appendix~\ref{sec:PD} for details.
In practice, getting a good estimate for
$z$ which solves~\eqref{eq:ProblemDiscrete} at each iteration requires few steps, since we initialize the variables of the algorithm as the variables from the last step of the previous iterate (and the measure $\mu$ does not change much from one step to the next). 
Moreover, we use simple projection steps and an adapted choice of the step sizes.
{The implementation
is developed in Python and
does not use
parallelization (outside of the default Numpy multi-threading). We indicate in a few
captions the wall time of the execution
on an average laptop.}

The success of the iterative bi-level algorithm relies upon the quality of the initialization of $\mu$: in the case of an image with many curves to reconstruct, in order to make sure to capture all of them, one may need to initialize $\mu$ with a very big set of pairs of discrete masses of intensity $+1$ and $-1$, randomly over the whole image, such that each two Dirac masses of a pair are close (e.g.\ in each other's $4 \times 4$ neighborhood). For the result presented in Fig.~\ref{fig:resultat_chromosomes}, $\mu$ was initialized as 2000 such pairs of Dirac masses, while there are only 42 chromosomes in the original image\footnote{Initially from~\cite[Fig.~2]{chromosomes}, with a rescaling, a change of contrast and to which noise was added}.

When $\mu$ is initialized as a very large number of Dirac pairs, some additional processing is required.
First, at each iteration, whenever two Dirac masses of opposite coefficients are in each other's $3 \times 3$ neighborhood,
it is considered that they should merge and they are removed from $\mu$.
Furthermore, after the Dirac masses have converged, it is important to post-process them. In practice, many pairs of Dirac masses may converge close to the endpoints of the same acquisition of a curve in the original image: they need to be merged. In order to do this, the curves need to be deduced from the discrete vector field $z \in \mathcal{V}_d$ (\textit{i.e.}, $z$ is decomposed into a superposition of open curves whose endpoints are the Dirac masses). Then, every time two pairs of opposite Dirac masses correspond to curves that superpose and are roughly of the same orientation, these curves are essentially merged, keeping only the two Dirac masses (of opposite signs) which correspond to the endpoints of the global curve, and discarding the other two. With the resulting set of Dirac masses,~\eqref{eq:ProblemDiscrete}~is then computed one last time, and the resulting $z$ is the final result. In addition to Fig.~\ref{fig:resultat_chromosomes}, we show
three results on the ``comma'' image on Fig.~\ref{fig:grille_resultats}: one with a random initialization
with two endpoints, and two with a random initialization with 15 pairs
of Dirac masses in the region of low potential $g$. The last one,
obtained after slightly blurring the potential, allows to recover nicely
and completely the main curve in the image.

\section{Curvature penalization}
\label{sec:curvature}
\subsection{Lifting of curvature-dependent energies}\label{sec:lifting}

We now detail how the method is extended
to incorporate curvature penalization,
using the celebrated representation in the group
of ``roto-translations'' 
introduced and developed in~\cite{SartiCitti,Boscainetal14,Prandietal,Duitsetal14,BekkersDuitsMS2015,bekkers_multi-orientation_2014,duits_optimal_2018,laville:hal-05124672} (and many other works by the same groups).
Yet as before, instead of relying on oriented
edge detection~\cite{bekkers_multi-orientation_2014} or solving an anisotropic eikonal equation to minimize the length
of curves~\cite{duits_optimal_2018}, we build a convex energy of
vector fields with divergence constraints.
We follow the framework in~\cite{ChambollePockRoto},
which addresses the implementation of
a ``total roto-translational variation'',
that is an energy which enforces
a curvature penalization of the gradient of a function. Up to a $90^\circ$ rotation,
in 2D, this is equivalent to penalizing
the mass of a zero-divergence field,
and we can re-use much of the numerical approach
of~\cite{ChambollePockRoto} for our implementation.
Adaption to higher dimension would be possible
(and not very difficult), yet computationally hardly
tractable.

We want to penalize  a planar curve,
parameterized by a Lipschitz
map $\gamma:[0,1]\to R$ where $R$ is
the image domain (a closed, arcwise connected
set, in general a rectangle), by a
curvature dependent energy of the form:
\begin{equation}\label{eq:curvatureenergy}
\int_0^1 g(\gamma(t))f(\kappa_\gamma(t))|\gamma'(t)|dt,
\end{equation}
for $f$ an even, convex function,
with $f\ge 1$, and $g$
as before a (positive, contiuous) varying weight.
Here, $\kappa_\gamma(t)$ is the (absolute) curvature,
that is, the length of the component of ${\gamma''}/|\gamma'|$ orthogonal to $\gamma'$
(it is given by $|\gamma''|$ if $\gamma$ is parameterized so that $|\gamma'|=1$ a.e.~in $[0,1]$).

In practice, we lift the image domain
$R$ by adding
a variable corresponding to the orientation
of the curves, setting $Q=R\times \Sp^1$ where  $\Sp^1:=\R/(2\pi \Z)$ is the periodic circle
of length $2\pi$. By convention, we
will denote $(x,\theta)$ a point in $R\times \Sp^1$
and $\vec{\theta}= (\cos\theta,\sin\theta)$ the
point in the circle defined by the angle $\theta$.

In~\cite{ChambollePockRoto}, it is shown that
 --- following the framework in~\cite{SartiCitti,Boscainetal14,Prandietal,Duitsetal14,BekkersDuitsMS2015,bekkers_multi-orientation_2014,duits_optimal_2018} --- one can
 represent~\eqref{eq:curvatureenergy}
 by lifting the curve $\gamma: [0,1]\to R$
as a curve $\Gamma(t):=(\gamma(t),\theta(t)): [0,1]\to Q$
where $\theta$ is the angle of $\gamma(t)'$ with the direction $(1,0)$ (so that $\Vec{\theta} =\gamma'/|\gamma'|$ {and the  curvature is $\kappa_\gamma =\theta'/\abs{\gamma'}$}), and weighting
$\Gamma$ with an appropriate ``length'',
as we describe now.
We introduce $f^\infty(t) = \lim_{a\to +\infty}f(at)/a$, the recession function of $f$ at infinity. Then, we define $\bar h$ as the perspective function of $f$: (\textit{cf}~\cite[eq.~(3)]{ChambollePockRoto}):
\begin{equation}
    \bar h(s,t) = \begin{cases}
        s f(t/s) & \text{ if } s>0,\\
        f^\infty(t) & \text{ it } s=0,\\
        +\infty  & \text{ else.}
    \end{cases}
\end{equation}
Equivalently, $\bar h$ is the support function
of $\{(a,b)\in\R^2: f^*(a)+b\le 0\}$, that is:
\begin{equation}\label{eq:dualbarh}
\bar h(s,t) = \sup\left\{ as + bt: f^*(a)+b\le 0\right\}.
\end{equation}
Then, for $\theta$ an angle (with $\vec{\theta}$ the
associated unit vector), and $v=(v^x,v^\theta)\in\R^2\times \R$, we
let:
\[
h(\theta,v) =\begin{cases} \bar h(\lambda,v^\theta) 
& \text{ if } v^x = \lambda\vec{\theta} \text{ for some }\lambda\in\R\,,\\
+\infty & \text{else.}
\end{cases}
\]
In particular, 
\begin{equation}\label{eq:dualdefh}
    h(\theta,v) = \sup\left\{ \vec{a}\cdot v^x
    + b v^\theta : (\vec{a},b)\in\R^2\times \R, f^*(\vec{a}\cdot\vec{\theta})+b\le 0\right\}.
\end{equation}      

One can check that with such a definition,
\begin{equation}\label{eq:rotocurveenergy}
\int_0^1 g(\Gamma^x(t))h(\Gamma^\theta(t),{\Gamma}'(t)) dt
= \int_0^1 g(\gamma(t)) f(\kappa_{\gamma}(t))
|\gamma'(t)|dt.
\end{equation}
The curvature naturally pops out because the vertical
component ${\Gamma^\theta}'$ of $\Gamma'$ is precisely
the derivative of the direction of the curve. We refer to~\cite{ChambollePockRoto} for the computational details.

Now, to adapt our setting to this framework, we consider normal charges $z=(z^x,z^\theta)$ supported in $Q$, and their horizontal projection (or marginalization) $\tilde z$ onto $R$ defined
by  $\tilde z(B):= z^x(B\times \Sp^1)$ for any borel set
$B\subset R$. Observe that if
\[
-\Div z = \delta_b - \delta_a
\]
(3D divergence)
in $Q$, then, denoting $a=(a^x,a^\theta)\in R\times \Sp$, etc, we find that:
\[
-\Div \tilde z = \delta_{b^x}-\delta_{a^x},
\]
in particular a curve in $Q$ between
the endpoints $a,b$ is projected onto a curve
in $R$ with endpoints $a^x,b^x$.
As before, given $\mu^+$, $\mu^-$ two measures
in $Q$ with $\mu^+(Q)=\mu^-(Q)$ (in practice, 
we use sums of Dirac masses), we can solve
the problem:
\begin{equation}\label{eq:RTenergy}
\min \left\{\int_Q g(x)h(\theta,z) : -\Div z = \mu^+-\mu^-\right\}.
\end{equation}
Thanks to Theorem~\ref{prop:smirnov} (provided $g$ and $R$ are such that (A0--2) hold), there
exists a measure $\sigma$ which decomposes $z$
into curves $\Gamma$ of finite length,
with properties~(\ref{eq:smirnovdecompoA}--\ref{eq:smirnovdecompoD}). In particular, $\sigma$-a.e.~curve
$\Gamma$ is minimal for the energy~\eqref{eq:rotocurveenergy},
so that its horizontal projection is minimal for~\eqref{eq:curvatureenergy}.
If $\mu^\pm$ are given by $n$ Dirac masses, $n\ge 1$,
we find that in general (up to a possible
non-uniqueness of the geodesic curves), the
horizontal projection $\tilde z$ consists therefore in
$n$ (possibly overlapping) curves minimizing the curvature dependent
energy~\eqref{eq:curvatureenergy} among all curves
joining the points in the support of $\mu^-$
to the points in the support of $\mu^+$.

\subsection{Numerical implementation}\label{sec:curvature_implementation}

In this new context, the discretization of~\eqref{eq:RTenergy} has the form
\[
\min_{z\in\mathcal{V}_d^{RT}, D^*z = \mu}
\sum_{i,j,k} g_{i,j} h(\theta_k,(Az)_{i,j,k})
\]
where $\mathcal{V}_d^{RT}\simeq (\R)^{(N-1)\times M\times K}
\times (\R)^{N\times (M-1)\times K}
\times (\R)^{N\times N\times K}$
is periodic in the last component and
represents the space of charges
in $Q=R\times \Sp^1$, discretized on a $N\times M \times K$ grid, $N,M,K\ge 1$, and $\theta_k = 2k\pi/K$. The operator $D$ is still a discrete
differentiation operator (with adjoint $D^*z$),
with the third component given by:
\[
D u_{i,j,k+\frac{1}{2}} = \begin{cases} u_{i,j,k+1}-u_{i,j,k}
& k = 1\dots,K-1 \\
 u_{i,j,1}-u_{i,j,K}
& k = 0.
\end{cases}
\]

The implementation relies again on 
optimizing the 3D version of~\eqref{eq:saddlePD}:
\begin{equation}\label{eq:saddleRT}
\min_{z: D^* z = \mu} \max_{p \in C_{g,h}} \sca{p}{Az}
\end{equation}
and in order to implement our algorithm, one now needs to be able to project on the constraint
set $C_{g,h}$ whose support function is
the energy. Thanks to~\eqref{eq:dualdefh}, this is
given by:
\begin{multline*}
C_{g,h} = \Big\{p_{i,j,k}=g_{i,j}(\vec{q}^x_{i,j,k},q^\theta_{i,j,k})\in (\R^3)^{N\times M\times K}:\\
f^*(\vec{q}^x_{i,j,k}\cdot \vec{\theta}_k) + q^\theta_{i,j,k} \le 0
\, \forall i,j,k\Big\}.
\end{multline*} 
Here as before, we have defined
the horizontal vector
$\vec{\theta}_k = (\cos \theta_k,\sin\theta_k)$.
Since each component $(i,j,k)$ is independent, the
projection onto $C_{g,h}$ is built upon the projection
onto the set whose $\bar h$ is the support function, \textit{cf}~\eqref{eq:dualbarh}:
\[
C_{\bar h} : = \left\{ (a,b)\in\R^2: f^*(a)+ b\le 0\right\}.
\]
Then, to obtain the projection
$q=\Pi_{C_{g,h}}(p)$ of a dual variable $p$
onto $C_{g,h}$, we compute for each coordinate $(i,j,k)$:
\begin{equation}\label{eq:anisoproj}
\begin{aligned}
(a,b) &= g_{i,j}\Pi_{C_{\bar h}}
\Big(\frac{1}{g_{i,j}}(\vec{p}_{i,j,k}^x\cdot \vec{\theta}_k,p_{i,j,k}^\theta)\Big)\\
\vec{q}^x_{i,j,k} &= \vec{p}^x_{i,j,k} + 
\left(a -\vec{p}_{i,j,k}^x\cdot \vec{\theta}_k\right)\vec{\theta}_k \\
q^\theta_{i,j,k} &= b.
\end{aligned}
\end{equation}

As in~\cite[Sec.~4.2]{ChambollePockRoto}, we considered
the following choices for $f$ (where each time
$\alpha$ is a positive parameter): \begin{itemize}
    \item $f_1(t) = 1+\alpha |t|$ (``TAC'', total absolute
    curvature);
    \item $f_2(t) = \sqrt{1+\alpha^2 t^2}$ (``TRL'', total
    roto-translational length---TRV in~\cite{ChambollePockRoto});
    \item $f_3(t) = 1+\alpha^2|t|^2$ (``EL'', Elastica energy---TSC in~\cite{ChambollePockRoto}).
\end{itemize}
The first and second allow for brutal change of
directions in the curves (which are penalized by the
turning angle), while the third enforces smoothness of
the curves.
The corresponding sets $C_{\bar h}$ are given by:
\begin{equation*}
    \begin{aligned}
    C_{\bar h_1} & = \left\{ (a,b) \in\R^2\,:\, a\le 1\,, |b|\le\alpha\right \}, \\
    C_{\bar h_2} & = \left\{ (a,b) \in\R^2\,:\,
    \max(0,a)^2 + (b/\alpha)^2 \le 1\right \}, \\
    C_{\bar h_3} & = \left\{ (a,b) \in\R^2\,:\, a+ b^2/(2\alpha)^2 \le 1 \right \}.
    \end{aligned}
\end{equation*}
The projections are implemented as in~\cite{ChambollePockRoto},
actually we re-used the \texttt{c++} programs developed by
Thomas~Pock for this purpose. While the projection onto
$C_{\bar h_1}$ is straightforward, the projections onto
$C_{\bar h_2}$ (for $\alpha \neq 1$) and $C_{\bar h_3}$
require to solve a non-linear problem and rely on a few
iterations of a Newton method. We refer to~\cite[Sec.~4.2]{ChambollePockRoto} for details.
\begin{figure}[htb]
    \centering
    \includegraphics[width=0.225\textwidth]{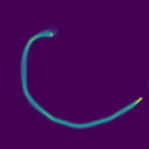}
    \includegraphics[width=0.3\textwidth]{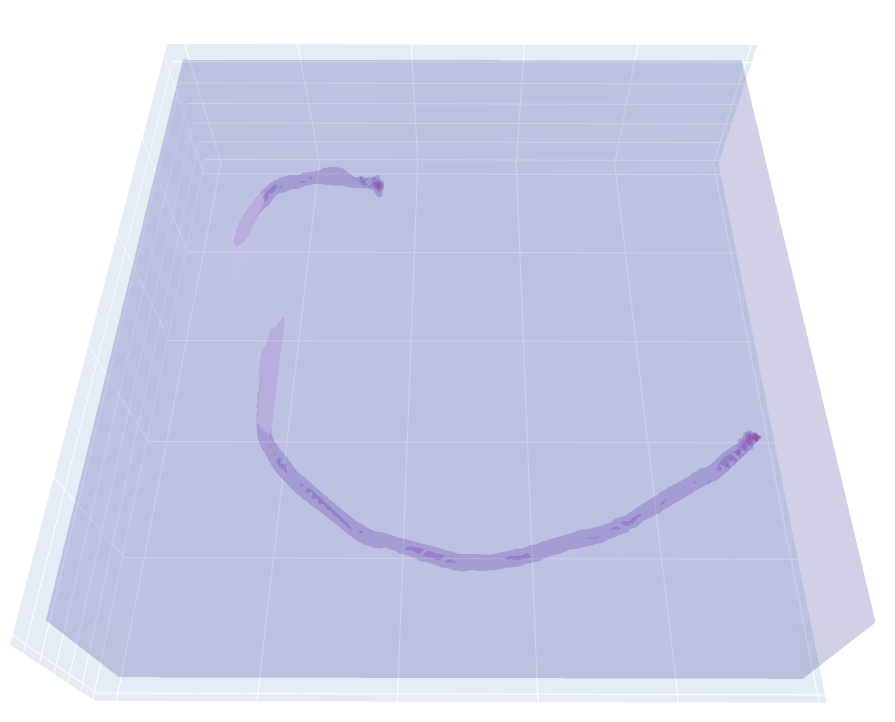}
    \includegraphics[width=0.3\textwidth]{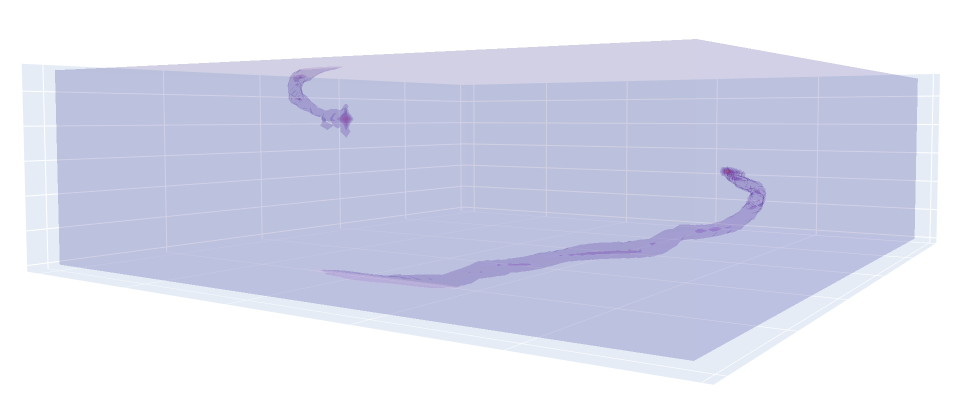}
    \caption{Minimization of the EL energy with $\alpha = 1$, on a $100 \times 100$ version of the ``comma'' image (note that the image is half the size here compared to previous examples). Left: projection on the image domain, center and right: result in the lifted image domain. This was obtained after 1500 primal-dual steps, with fixed endpoints, for a total of 4 min.\ 50s.\ wall time on an average laptop.}
    \label{fig:resultat_virgule_el}
\end{figure}

\subsection{Rules for moving the endpoints}\label{sec:Iterative_curvature}

We also adapted the method of Section~\ref{sec:Iterative} to
find the endpoints of a family of curves. In that case,
the endpoints are pairs $(x,\theta)$ with $x=(i,j)$
and $\theta = 2\pi k/K$ for some integers $(i,j,k)$.
The update of $(i,j)$ follows the same rules as
in Section~\ref{sec:EndPoints}. The update of $k$
is simpler:  at the free endpoints, we would like
the curvature to vanish 
and the curves to satisfy ${\Gamma^\theta}'(0)={\Gamma^\theta}'(1)=0$. Hence
if $z_{i,j,k\pm \frac{1}{2}}$ is positive (up
to some threshold), we move
up or down (depending whether the point is in $\mu^+$---leaving
or $\mu^-$---entering, which is given by the sign of
the Dirac mass) the Dirac mass at $(i,j,k)$, while
if negative we move it in the other direction.

\begin{figure}[htb]
    \centering
    \includegraphics[width=0.25\textwidth]{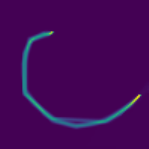}
    \includegraphics[width=0.25\textwidth]{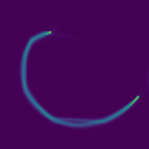}
    \caption{Left:  TAC, right: EL, on the $100 \times 100$ ``comma'' image, with parameter $\alpha=5$, 1500 primal-dual steps with fixed endpoints. Observe that the TAC allows for
    jumps in the direction, which are prevented by
    the EL.}
    \label{fig:resultat_virgule_tac_el}
\end{figure}
\subsection{Numerical experiments}
We show a few numerical results for curvature-dependent energies. These results require more processing time than curvature-independent implementations, since the underlying computations are in the 3D lifted variables. In every example in this section, the angle $\theta$ is discretized into 30 possible values (with the notation of Section~\ref{sec:curvature_implementation}, $K=30$). We consider the energies described earlier,
namely the ``TAC'', the ``TRL'' and the ``EL''. We also show some results in the 3D lifted domain, where the length and width of the volume represent the domain of the image, and the height represents the angle $\theta$, as detailed in Section~\ref{sec:lifting}.
We start by looking at a simple example,
recovering the previous  ``comma'' shape with the EL
energy, see Figure~\ref{fig:resultat_virgule_el}. 
This energy yields a very smooth curve. Then
in Figure~\ref{fig:resultat_virgule_tac_el} we compare
the TAC and the EL, with a stronger curve penalization ($\alpha = 5$): the curve obtained with EL is clearly smoother than the one obtained with TAC, which allows for jumps in the direction.

We then compare in Figure~\ref{fig:resultats_croix}
the three different curvature-dependent energies,
on a simple image with two crossing curves
(with $\alpha=1$ and fixed endpoints).
The smaller image on the bottom right 
is the result obtained with the curvature-independent energy~\eqref{isotropic_2d}
(after 1800 primal-dual steps).
Observe that in the latter case, the crossing is
obviously not recovered, and the method finds
a combination of many geodesic curves with almost
same length which lie in the wide
the low-potential area. On the other hand, the curvature
penalized results are much more stable.
\begin{figure}[htb]
    \begin{center}
    \includegraphics[width=0.35\textwidth]{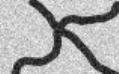}
    \includegraphics[width=0.35\textwidth]{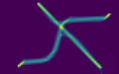}
    \makebox[.2\textwidth]{} 
    \\[.5mm]
    \includegraphics[width=0.35\textwidth]{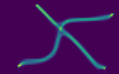}
    \includegraphics[width=0.35\textwidth]{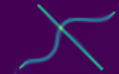}
    \includegraphics[width=0.2\textwidth]{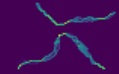}
    \end{center}
    \caption{Top left: $80 \times 50$ pixel input image. Top right: TAC. Bottom left: TRL, center: EL. Small image: without curvature
    penalization. We used $\alpha = 1$, and 1800 primal-dual steps (with fixed endpoints).
    Wall time of roughly 4 min.\ 30 s.\ for every experiment, on an average laptop.
    The crossing is better reconstructed with the EL, which
    totally prevents part of the flow to make an abrupt
    turn towards a wrong direction.}
    \label{fig:resultats_croix}
\end{figure}
The resulting curves are smoother for the TRL energy than for the TAC energy, and smoothest for the Elastica EL energy. Also, the two curves are not perfectly identified
with TRL and TAC, since curves turning abruptly still
have bounded energy---while they have infinite EL energy. 
This is better seen in the lifted image domain, see Figure~\ref{fig:resultats_croix_2}.

\begin{figure}[htb]
    \centering
    \includegraphics[width=0.3\textwidth]{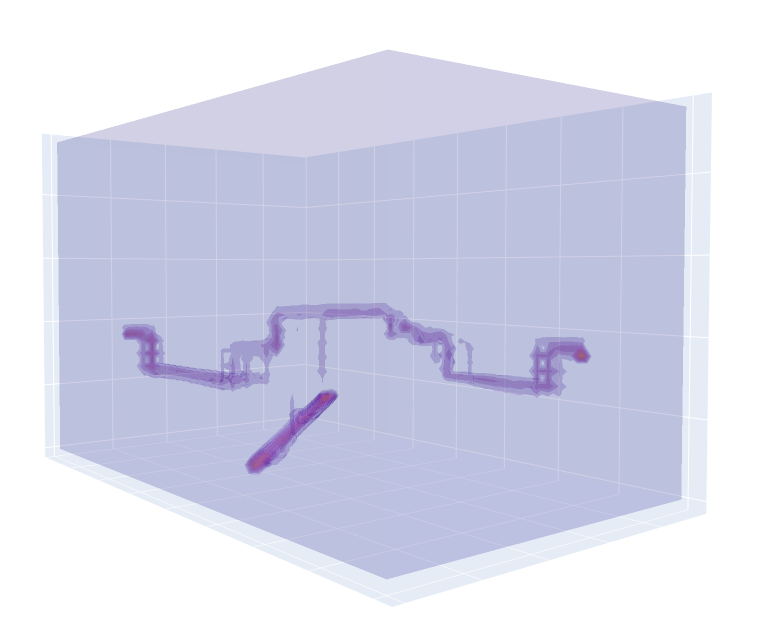}
    \includegraphics[width=0.3\textwidth]{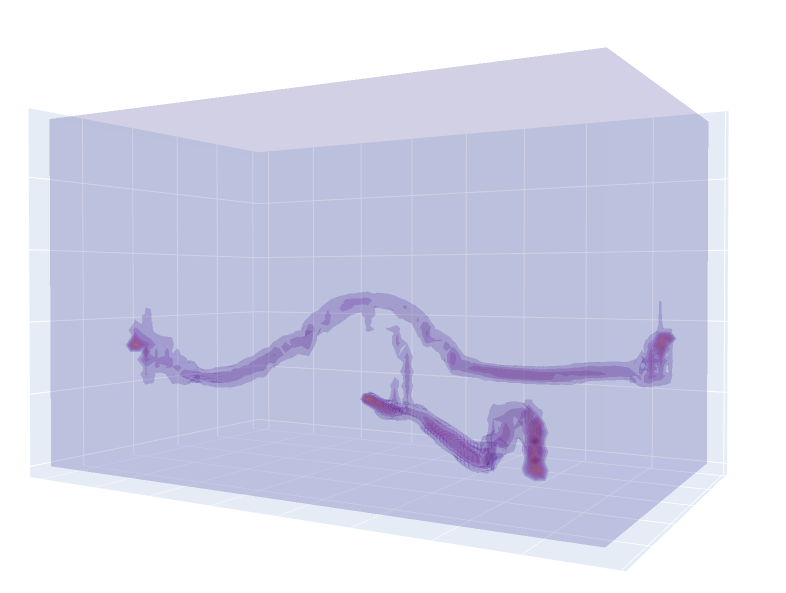}
    \includegraphics[width=0.3\textwidth]{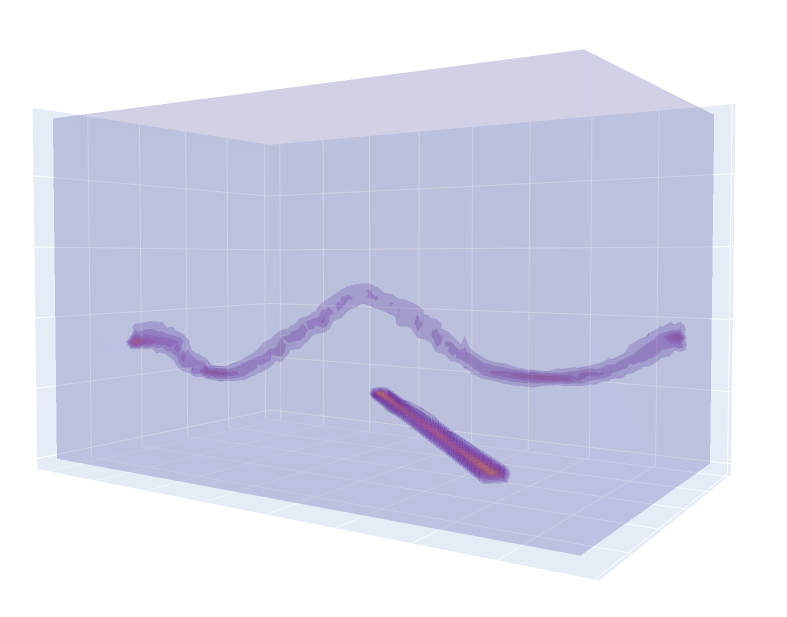}
    \caption{From left to right:
    results of Figure~\ref{fig:resultats_croix} represented in the lifted image domain for the TAC, the TRL and the EL.}
    \label{fig:resultats_croix_2}
\end{figure}


We then present a result where the endpoints are not fixed, using the rules of Section~\ref{sec:Iterative} extended to the case of curvature-dependent energies (\textit{cf.}~Sec.~\ref{sec:Iterative_curvature}), to evolve the endpoints; see
Figure~\ref{fig:resultat_faux_chromosomes}. 
Now, the curves can cross each other (unlike in Figure~\ref{fig:resultat_chromosomes}) so that
curvature-dependent energies are more appropriate.
In this experiment the iterative algorithm recovers all of the curves in the image.
\begin{figure}[htb]
    \centering
    \includegraphics[width=0.3\textwidth]{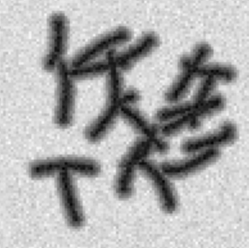}
    \includegraphics[width=0.3\textwidth]{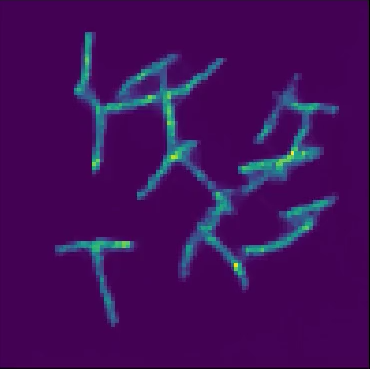}
    \caption{Final result of the iterative algorithm, on a $100 \times 100$ image, using the TRL energy with parameter $\alpha = 1$, after 10 iterations of the algorithm of Section~\ref{sec:Iterative}, each iteration corresponding to 500 primal-dual steps, for a total of about 5 min.\ wall time on an average laptop. The Dirac masses were initialized as 1000 pairs of opposite intensities randomly chosen in the sublevel set $\{g \leq g_{max}\}$.}
    \label{fig:resultat_faux_chromosomes}
\end{figure}

\section{Conclusion}
In this paper, we have introduced a variational method for curve extraction which differs from the usual Eikonal equation approach~\cite{Peyreetal}, as we represent paths as measure vector fields.
It allows to simultaneously compute geodesics between a large numbers of endpoints.
We also proposed a bi-level method to automatically fit the endpoints, yielding a complete, mostly unsupervised automatic curve extraction algorithm. It was extended
to consider curvature penalization, allowing to
reconstruct smooth or crossing curves.
Further work could focus on more refined averaging
operators to improve the sharpness and isotropy
of the results, building upon recent findings in the context of discrete total variations. 

{\section*{Acknowledgements} A.~Chambolle acknowledges the support of the ``France 2030'' funding ANR-23-PEIA-0004 (``PDE-AI'').
Part of this work was done while M.~Arthaud was an intern in the Mokaplan team, 
with the support of INRIA Paris. We thank Thomas Pock for
allowing to re-use the code developed for~\cite{ChambollePockRoto}.}

%
%
%
\bibliographystyle{splncs04}
\bibliography{biblio.bib}
\appendix

\section{Proof of Theorem~\ref{prop:smirnov}}
\label{apx:decompo}
We first provide the proof of Theorem~\ref{prop:smirnov}, before stating some Lemmas on which it relies.

\begin{proof}[of Theorem~\ref{prop:smirnov}]
We begin by proving the existence of a solution $z$.
By Lemma~\ref{lem:feasible} below, Problem~\eqref{eq:ProblemContinuous2} is feasible, hence we may consider a minimizing sequence $(z_{n})_{n\in \N}$. There exists a constant $C>0$ such that for all $n\in \N$,
\begin{align*}
	C \geq \int g(x,z_n) \geq c \abs{z_n}(\R^d),
\end{align*}
where we have used Assumption (A2) in the second inequality.
Thus $(z_{n})_{n\in \N}$ is bounded in $\mathcal{M}(\R^d;\R^d)$ and we may extract a subsequence $(z_{n})_{n\in I}$ (with $I \subseteq  \N$ infinite) which converges in the weak-* topology to some  $z \in \mathcal{M}(\R^d;\R^d)$. The constraint $-\Div \tilde{z} = \mu$ is weak-* closed, hence $-\Div z=\mu$. Moreover the lower-semi-continuity of $G\colon \tilde{z}\mapsto \int g(x,\tilde{z})$ for the weak-* convergence (see Section~\ref{sec:spacecharges}) implies that
\begin{align*}
	G(z)\leq \liminf_{\substack{n \to +\infty,\\ n\in I}} G(z_n) = \inf\eqref{eq:ProblemContinuous2}.
\end{align*}
As a result, $z$ is a solution to~\eqref{eq:ProblemContinuous2}.

Now, we prove the existence of a complete decomposition of for $z$. Smirnov's theorem~ (Theorem~\ref{thm:smirnov}) implies that  $z$ completely decomposes into two parts, $p$ and $q$, with $\Div p=0$ and $q$ is decomposed on curves of finite length in the support of $q$.
By Lemma~\ref{lem:pplusq} below, this implies that
\begin{align*}
 \int g(x,z) = \int g(x,p) + \int g(x,q)\geq \int g(x,q),
\end{align*}
and the minimality of $z$ for~\eqref{eq:ProblemContinuous2} implies  that
$p=0$ and $z=q$. Hence, the same theorem ensures that $z$ can be completely decomposed as in~(\ref{eq:smirnovdecompoA},~\ref{eq:smirnovdecompoB},~\ref{eq:smirnovdecompoC}) for some Borel measure $\sigma$. Moreover,  $\sigma$-a.e. curve lies in $\spt z \subseteq Q$.

Our next step is to prove~\eqref{eq:smirnovdecompoD}.
This follows from~\cite[Lemma 5.2]{ChambollePockRoto} (for the case where $g$ is continuous, which may be extended to l.s.c. $g$ by a monotone convergence argument in \cite[Cor. 5.4]{ChambollePockRoto}). We propose below an alternative proof which relies on Lemma~\ref{lem:pplusq} and on Fatou's Lemma.
Let us write
\begin{align*}
	\Gamma= \left\{\gamma \in \Lipd :\forall t \in  \R,\  \gamma(t) \in Q \right\},
\end{align*}
so that $\sigma(\Lipd\setminus \Gamma)=0$.
For every Borel partition $(\Gamma_1,\Gamma_2)$ of $\Gamma$, the decomposition formula~\eqref{eq:smirnovdecompoB} together with Lemma~\ref{lem:pplusq} yield
\begin{align*}
	\int g(x,z) &= \int g\left(x, \int_\Gamma\zgamma\diff \sigma(\gamma) \right) \\
	&= \int\! g\left(x, \int_{\Gamma_1}\zgamma\diff \sigma(\gamma) \right) + \int g\left(x, \int_{\Gamma_2} \zgamma\diff \sigma(\gamma) \right).
\end{align*}
By induction, if $(\Gamma_i)_{i \in I}$, with $I \subseteq \N$, is a Borel partition of $\Gamma$, we obtain similarly
\begin{align}\label{eq:sumG}
	G\left(\int_\Gamma\zgamma\diff \sigma(\gamma)\right) = \sum_{i \in I} G\left(\int_{\Gamma_i}\zgamma\diff \sigma(\gamma)\right)
\end{align}
with the notation $G(z) = \int g(x,z)$. Now, let $n \in \N$. Since $\Gamma$ is a compact metric space (see Section~\ref{sec:smirnovdecompo}),
 there exists a Borel partition $(\Gamma^n_i)_{i \in I^n}$ such that $\mathrm{diam}(\Gamma^n_i)\leq 1/(n+1)$ for all $i\in I^n$. Moreover, let us define, for all $\ell \in \N$,
\begin{align}
	B_\ell &= \left\{\gamma \in \Gamma : \mathrm{spt}\gamma' \subseteq [-\ell, \ell] \right\},\label{eq:defbpartition}\\
	\mbox{and } \Gamma^n_{i,\ell}&= \Gamma^n_i \cap \left(B_\ell\setminus B_{\ell-1}\right)\nonumber
\end{align}
with $B_{-1}=\emptyset$. We obtain a Borel subdivision\footnote{There are curves in $\Lipd$ which are not in any $B_\ell$, but Smirnov's construction explicitely charges the set of curves such that $\spt \gamma' \subseteq [a,b]$ for some $a,b \in  \R$ (see \cite[Sec. 2.8.2]{smirnov}). As a result $\sigma\left(\Gamma\setminus \bigcup_\ell B_\ell\right)=0$.} of $\Gamma\cap \left(\bigcup_{\ell} B_\ell\right)$, and possibly discarding the indices such that $\sigma(\Gamma_{i,\ell}^n)=0$,  we may rewrite~\eqref{eq:sumG} as
\begin{align*}
	G\left(\int_\Gamma\zgamma\diff \sigma(\gamma)\right)= \sum_{i,\ell} \sigma(\Gamma_{i,\ell}^n)G\left(\frac{1}{\sigma(\Gamma_{i,\ell}^n)}\int_{\Gamma_{i,\ell}^n}\zgamma\diff \sigma(\gamma)\right) = \int_{\Gamma} \tilde{G}^n(\gamma)\diff \sigma(\gamma),\\
	\mbox{where}\	\tilde{G}^n(\gamma) := G\left(\frac{1}{\sigma(\Gamma_{i,\ell}^n)}\int_{\Gamma_{i,\ell}^n}\zgammat\diff \sigma(\tilde{\gamma})\right) \mbox{ for $i,\ell$ such that $\gamma \in  \Gamma_{i,\ell}^n$.}
\end{align*}
Fatou's lemma then implies that
\begin{align*}
	G\left(\int_\Gamma\zgamma\diff \sigma(\gamma)\right) = \liminf_{n\to +\infty} \int_{\Gamma} \tilde{G}^n(\gamma)\diff \sigma(\gamma)
	\geq \int_{\Gamma} \liminf_{n\to +\infty}\left(\tilde{G}^n(\gamma)\right)\diff \sigma(\gamma).
\end{align*}
On the other hand, arguing as in \cite[Sec. 2.4.2]{smirnov}, we note that, restricted to each $B_\ell$, the map $\gamma\mapsto \zgamma $ is continuous for the weak-* topology. The fact that $\mathrm{diam}(\Gamma^n_{i,\ell})\leq 1/(n+1)$ thus implies that for $\sigma$-a.e. $\gamma \in  \Gamma$, if for each $n$, $\Gamma^n_{i,\ell}$ is the cell which contains $\gamma$, 
\begin{align*}
	\frac{1}{\sigma(\Gamma_{i,\ell}^n)}\int_{\Gamma_{i,\ell}^n}\zgammat\diff \sigma(\tilde{\gamma}) \overset{\ast}{\rightharpoonup} z_\gamma, 
\end{align*}
so that $\liminf_{n\to +\infty} \tilde{G}^n(\gamma) \geq G(\zgamma)$ by the lower semi-continuity of $G$.
This yields
\begin{align*}
		G\left(\int_\Gamma\zgamma\diff \sigma(\gamma)\right)\geq \int_\Gamma G(\zgamma)\diff \sigma(\gamma),
\end{align*}
and the converse inequality follows from Jensen's inequality. As a result, \eqref{eq:smirnovdecompoD} holds.


To prove that $\sigma$-a.e. is a geodesic, we use a measurable selection result: there is a Borel map $\rho$ which maps any $(x,y)\in Q^2$ to some geodesic (in the sense of point \textit{(\ref{point1})} in the statement) $\rho_{x,y}\colon [0,1]\to Q$ with $\rho_{x,y}(0)=x$, $\rho_{x,y}(1)=y$. 
That result follows from~\cite[Prop. 2.3.5]{LOHMANN2022739} (with the minor difference that they state it for isotropic functionals $g(x,z)=g(x)$), but our Assumption (A0) makes the proof more straightforward, see Lemma~\ref{lem:selection} below.

For each $\gamma \in \bigcup_{\ell} B_\ell$, the map $\gamma \mapsto b(\gamma)$ (resp. $\gamma\mapsto e(\gamma)$) which associates each curve to its beginning point (resp. endpoint) is Borel. We define
\begin{equation*}
    \tilde{z}=\int z_{\rho_{b(\gamma),e(\gamma)}}d \sigma(\gamma),
   \quad {\mbox{that is, } \int \varphi \cdot d \tilde{z}=\int \langle z_{\rho_{b(\gamma),e(\gamma)}},\varphi\rangle d\sigma(\gamma)}
\end{equation*}
for all $\varphi\in C_c(\R^d;\mathbb{R}^d)$,
one may check that $\Div \tilde{z}=\Div z$ and
\begin{align*}
    \int g(x,\tilde{z})&\le \int \int g(\rho_{b(\gamma),e(\gamma)}(t),\rho_{b(\gamma),e(\gamma)}'(t)) dt\, d \sigma(\gamma)\\ &\le \int \int g(\gamma(t),\gamma'(t)) dt\, d \sigma(\gamma) = \int g(x,z).
\end{align*}
By optimality of $z$, each inequality must be an equality, hence $\sigma$-almost every $\gamma$ is a minimizing curve.

Eventually, introduce $S:= \int \delta_{b(\gamma)}d\sigma(\gamma)$, $R:=\int \delta_{e(\gamma)}d\sigma(\gamma)$. Both are non-negative measures, with
\begin{equation*}
    S-R= \int \Div \zgamma d \sigma(\gamma) = \Div z.
\end{equation*}
Hence, the Hahn-Jordan decomposition theorem ensures that $(\Div z)^+ \le S$ and $(\Div z)^-\le R$. Then, Equality~\eqref{eq:smirnovdecompoC} yields
\begin{equation*}
    (S+R)(Q) = \int \|\Div(\zgamma)\|_{\mathrm{TV}}d \sigma =\left\| \Div(z)\right\|_{\mathrm{TV}}  \! \! = \left((\Div z)^+ + (\Div z)^-\right)(Q)
\end{equation*}
so that $(\Div z)^+ = S$ and $(\Div z)^-= R$, and {point \textit{(\ref{point2})}} follows by considering their support.
\qed
\end{proof}

The following Lemma addresses the existence of a feasible point for Problem~\eqref{eq:ProblemContinuous2}.
\begin{lemma}\label{lem:feasible}
Under Assumptions (A0), (A1), (A2), with $\mu(Q)=0$,  Problem~\eqref{eq:ProblemContinuous2} is feasible.
\end{lemma}

\begin{proof}
Let $\mu=\mu^+-\mu^-$ be a Hahn-Jordan decomposition of $\mu$, and so that $\mu^+(Q)=\mu^-(Q)=\abs{\mu}(Q)/2$. Since the case $\mu=0$ is straightforward, we may assume, up to a rescaling, that $\mu^+(Q)=1$. For each $n\in \N$, there exist discrete measures of the form $\mu_n^+ =  \frac{1}{n}\sum_{i=1}^n \delta_{x_{i,n}}$, $\mu_n^- = \frac{1}{n}\sum_{j=1}^n \delta_{y_{j,n}}$ where $\{x_{i,n}\}_{i=1}^n\subseteq Q$, $\{y_{j,n}\}_{j=1}^n\subseteq Q$, and such that $\mu_n^+ \overset{\ast}{\rightharpoonup} \mu^+$, $\mu_n^- \overset{\ast}{\rightharpoonup} \mu^-$ as $n\to +\infty$.

For each pair $(x_i,y_j)$ we consider a curve $\gamma_{i,j}$ joining $x_i$ and $y_j$ satisfying ~\eqref{eq:arcwiseD}, and we let  $ z^n= \left(\sum_{i,j} z_{\gamma_{i,j}}\right)/n^2$. The convexity of $G\colon z\mapsto \int g(x,z)$ implies that
\begin{align*}
    D \geq \frac{1}{n^2} \sum_{i,j}G(z_{\gamma_{i,j}}) \geq  G(z^n)\geq c \abs{z^n}(\R^d)
\end{align*}
(where the first equality follows from (A0) and the last one from (A2)). As a result we may extract a subsequence $(z_n)_{n \in I}$ (with $I\subseteq \N$ infinite) such that $z^n \overset{\ast}{\rightharpoonup}z$ for some $z\in \mathcal{M}(\R^d;\R^d)$ as $n\to +\infty$ in $I$. Then, in the weak-* sense,
\begin{align*}
 \Div z = \lim_{\substack{n\to +\infty,\\n\in I}}\Div z^n &= \lim_{\substack{n\to +\infty,\\n\in I}} \left(\frac{1}{n}\sum_{i=1}^n \delta_{x_{i,n}} - \frac{1}{n}\sum_{j=1}^n \delta_{y_{j,n}}   \right)=\mu^+-\mu^- =\mu.
\end{align*}
Moreover,
\begin{align*}
     \abs{z}(\R^d\setminus Q) &\leq \liminf_{\substack{n\to +\infty,\\n\in I}}\abs{z^n}(\R^d\setminus Q) = 0,
\end{align*}
so that $\textrm{spt}(z) \subseteq Q$. Eventually, the lower semi-continuity of $G$ ensures that $G(z)\leq L<+\infty$, and $z$ is thus admissible for~\eqref{eq:ProblemContinuous2}. \qed
\end{proof}

The next lemma shows that the decomposition of a charge in two parts implies a similar decomposition of integral functionals.
\begin{lemma}\label{lem:pplusq}
Let $p,q \in \mathcal{V}$ such that
\begin{align}\label{eq:sumtvpq}
    \abs{p+q} = \abs{p}+\abs{q}
\end{align}
as measures, and $G(z)=\int g(x,z)$ such that (A1) holds. Then
\begin{align}
    \int g(x,p+q) = \int g(x,p) + \int g(x,q).
\end{align}
\end{lemma}

\begin{proof}
Let $z= p+q$. By assumption, $\abs{z}=\abs{p}+\abs{q}$, hence $p$ and $q$ are absolutely continuous w.r.t. $\abs{z}$. As a result of~\eqref{eq:sumtvpq}, for $\abs{z}$-a.e. $x \in \R^d$ their densities w.r.t. $\abs{z}$ satisfy
\begin{align*}
	\frac{\diff z}{\diff \abs{z}}(x) = \frac{\diff p}{\diff \abs{z}}(x)+ \frac{\diff q}{\diff \abs{z}}(x) \quad \mbox{ and } \abs{\frac{\diff z}{\diff \abs{z}}(x)} = \abs{\frac{\diff p}{\diff \abs{z}}(x)}+ \abs{\frac{\diff q}{\diff \abs{z}}(x)}.
\end{align*}
Hence $\frac{\diff p}{\diff \abs{z}}(x)$ and $\frac{\diff q}{\diff \abs{z}}(x)$ must be colinear and pointing in the same direction. Thus, there exists $\alpha \in  [0,1]$ such that
\begin{align*}
	\frac{\diff p}{\diff \abs{z}}(x) = \alpha \frac{\diff z}{\diff \abs{z}}(x) \quad \mbox{and}\quad 	\frac{\diff q}{\diff \abs{z}}(x) = (1-\alpha) \frac{\diff z}{\diff \abs{z}}(x).
\end{align*}
Then by the positive $1$-homogeneity of $g$ (w.r.t its second argument) we get
\begin{align*}
	g\left(x,\frac{\diff z}{\diff \abs{z}}(x)\right) &= \alpha g\left(x,\frac{\diff z}{\diff \abs{z}}(x)\right) + (1-\alpha) g\left(x,\frac{\diff z}{\diff \abs{z}}(x)\right)\\
						       &=  g\left(x,\frac{\diff p}{\diff \abs{z}}(x)\right)+g\left(x,\frac{\diff q}{\diff \abs{z}}(x)\right),
\end{align*}
so that $\int g(x,z) = \int g(x,p)  +\int g(x,q)$ (see~\eqref{eq:defcvxmes}).

\end{proof}

The next Lemma provides a measurable selection map for the geodesics.
\begin{lemma}\label{lem:selection}
	Assume that (A0), (A1), (A2) hold. Then, there is a Borel map $\rho\colon Q\times Q \rightarrow \Lipd$ which maps any $(x,y)\in Q^2$ to some geodesic (for $g$) from $x$ to $y$ in $Q$. More precisely, $\rho(x,y)=\rho_{x,y}$ with $\rho_{x,y}\colon [0,1]\to Q$ with $\rho_{x,y}(0)=x$, $\rho_{x,y}(1)=y$ and $\rho_{x,y}$ is a Lipschitz curve which minimizes 
	$\int g(\gamma(t),\gamma'(t))\diff t$ 
among the Lipschitz curves from $x$ to $y$ in Q).
\end{lemma}

\begin{proof}
The proof consists in invoking~\cite[Cor. 1]{Brown1973}. 
By Assumptions (A0) and (A2), for each $(x,y)\in Q$, there exists some curve $\gamma$ from $x$ to $y$, with 
\begin{align*}
    D\geq \int g(\gamma(t),\gamma'(t))\diff t \geq c \int \abs{\gamma'(t)}\diff t.
\end{align*}
Let us fix $\ell \in \N$, $\ell \geq D/c$. Possibly reparametrizing $\gamma$ using arclength, we deduce that for every $(x,y) \in  Q^2$,  there is some $\gamma \in B_\ell$ with $\gamma(-\ell)=x$, $\gamma(\ell)=y$, and $\int g(\gamma,\gamma')\leq D$, where $B_\ell$ defined in \eqref{eq:defbpartition}. 
The set $B_\ell$ is compact by the Arzelà–Ascoli theorem, hence it is a complete metric space, and so is $Q^2$.

Let $\mathcal{F}\colon \gamma \in B_\ell \mapsto \int g(\gamma,\gamma')\in \R_+\cup \{+\infty\}$.
We observe that $\mathcal{F}$ is lower semi-continuous. Indeed, let $(\gamma_{n})_{n\in \N}$ be a sequence which converges to some $\gamma \in B_\ell$ (that is, for the uniform convergence on $[-\ell,\ell]$). 
Possibly extracting a subsequence, we assume that $\lim_{n\to +\infty} \int g(\gamma_n,\gamma_n') = \liminf \int g(\gamma_n,\gamma_n')$.
Then, the functions $(\gamma_n')$ are bounded in $L^1(-\ell,\ell)$, hence we may again extract a subsequence (that is find $I \subseteq \N$ infinite) such that $\gamma_n'$ converges a.e. to some $\lambda \in L^1(-\ell,\ell)$ as $n\to +\infty$ in $I$.
Passing to the limit in the equality (using dominated convergence)
\begin{align*}
	\gamma_n(t) = \gamma_n(-\ell)+\int_{-\ell}^t \gamma_n'(s)\diff s
\end{align*}
we see that $\lambda=\gamma'$. Then, by Fatou's lemma, $\int g(\gamma,\gamma') \leq \liminf \int g(\gamma_n,\gamma_n')$, so that $\mathcal{F}$ is lower semi-continuous.

As a result, the set 
\begin{align*}
	\mathcal{D} = \left\{((x,y),\gamma)\in Q^2\times B_\ell: \mathcal{F}(\gamma) <+\infty\ \mbox{and $\gamma(-\ell)=x$, $\gamma(\ell)=y$}\right\}
\end{align*}
is Borel, and, for each $(x,y)\in Q$, the set $\mathcal{D}_{(x,y)}=\{\gamma: ((x,y),\gamma)\in \mathcal{D} \}$ is $\sigma$-compact (it is the union of $\{\mathcal{F}\leq n\}\cap \{\gamma: \gamma(-\ell)=x, \gamma(\ell)=y\}$ for $n \in \N$, which are closed in the compact set $B_\ell$).
As a result, \cite[Cor. 1]{Brown1973} ensures the existence of a Borel measurable selection map $\rho$ as claimed.
	
\end{proof}



\section{Algorithm for the convex optimal
path problem}\label{sec:PD}
To solve~\eqref{eq:ProblemDiscrete}
we consider the saddle-point
 formulation~\eqref{eq:saddlePD},
 in which $\|\cdot\|_*$ is the dual norm of $\|\cdot\|$, {defined by $\|q\|_* = \sup_{\|p\|\le 1} q\cdot p$,}
and $\chi$ denotes a characteristic function
($0$ if the condition is satisfied, $+\infty$ else). Given $\tau, \sigma > 0$ with $\tau \sigma \|A\|^2\le 1$, the algorithm in~\cite{chambolle_pock} is:
\begin{enumerate}
    \item Initialize $z^0 \in \mathcal{V}_d$, $p^0 \in (\mathbb{R}^\mathcal{N})^2$, set $\bar{z}^0 = z^0$
    \item For each iteration $n \geq 0$, update:
    \begin{equation*}
    \begin{cases}
        p^{n+1} &= \text{Proj}_{\{p:\forall i, \, j, \, \|p_{i,j}\|_* \leq g_{i,j} \}}(p^n + \sigma A \bar{z}^n)\\
        z^{n+1} &= \text{Proj}_{z:\{D^* z = \mu\}}(z^n - \tau A^* p^{n+1})\\
        \bar{z}^{n+1} &= 2 z^{n+1} - z^n
    \end{cases}
    \end{equation*}
\end{enumerate}
where the first step requires to project onto  $\{p:\forall i, \, j, \, \|p_{i,j}\|_* \leq g_{i,j} \}$, which amounts to project independently each component on a 2D disc, and
\begin{equation*}
\text{Proj}_{\{D^* z = \mu\}}(z) = z + D(D^*D)^{-1}D^*(z_0 - z)
\end{equation*}
is the projection onto the space ${D^*z = \mu}$, with $z_0$ being any vector field such that $D^* z_0 = \mu$. The discrete
Neumann  Laplacian $(D^*D)$ (which is invertible
on functions with zero average)
is diagonalized and inverted by means
of a DCT (which turns out to be consistent with the no-flux
condition on the boundary for vector fields), except in the
``roto-translational'' representation (Sec.~\ref{sec:curvature}) where the block is periodic in the third component, and
one has to use a FFT for that one. We relied on the Python
bindings \texttt{pyFFTW} for the FFTW3~\cite{FFTW3}
library to compute the Fourier transforms.
The extension described in Section~\ref{sec:curvature}
just requires to adapt the first projection step
to the corresponding anisotropic energy, replacing
it with the formulas~\eqref{eq:anisoproj}, for the
various choices of $\bar h$.

For bounded sets
\[B_1 \subset \{z \in \mathcal{V}_d, \, D^*z = \mu \} \quad\text{ and }\quad B_2 \subset \{p \in (\mathbb{R}^\mathcal{N})^2, \, \forall i,j, \, \|p_{i,j}\|_* \leq g_{i,j}\},\]
we consider the partial gap:
\begin{equation*}
    \mathcal{G}_{B_1 \times B_2}(z,p) = \max_{p' \in B_2} \langle A z, p' \rangle - \min_{z' \in B_1} \langle A z', p \rangle.
\end{equation*}
It is shown~\cite{chambolle_pock} that:
\begin{equation} \label{cvg rate}
    \mathcal{G}_{B_1 \times B_2}(\bar z_k, \bar p_k) \leq \frac{D(B_1,B_2)}{k}
\end{equation}
where we have denoted $\bar z_k = \sum_{i=1}^k z^i / k$ and $\bar p_k = \sum_{i=1}^k p^i / k$ and where
\begin{equation} \label{diameter}
    D(B_1,B_2) = \max_{(z,p) \in B_1 \times B_2} \frac{\|z - z^0\|^2_2}{2\tau} + \frac{\|p - p^0\|^2_2}{2\sigma}
\end{equation}
Since the $z$ variable is expected to be a superposition of curves, one expects $\|z - z^0\|^2_2 \sim \sqrt{N \times M}$ (it behaves like the characteristic length of the domain), and since the dual variable represents the (uniformly bounded) gray levels values of the whole image, $\|p - p^0\|^2_2 \sim N \times M$. Hence, optimizing for $\tau$, $\sigma$ in~\eqref{diameter} under the constraint $\tau \sigma \|A\|^2 \le 1$, we set in practice $\tau = .99/(\|A\| (N \times M)^{1/4})$ and $\sigma = (N \times M)^{1/4}/\|A\|$. With this choice, $D(B_1, B_2) \lesssim \|A\|(N \times M)^{3/4}$ (the norm of the averaging
operator $A$ is of order $1$).

For images of similar size as the one in Fig.~\ref{fig:resultat_chromosomes} ($200 \times 200$ pixels), only about 60 steps of the primal-dual algorithm are needed at each motion of the Dirac masses. Practical convergence is a lot faster than the $O(1/k)$ rate guaranteed in~\eqref{cvg rate}, as for a $200 \times 200$ image, that would require thousands of iterations (estimating $D(B_1, B_2) \sim 3000$). It seems
that the ``preconditioning'' induced by the choice of projecting onto the divergence constraint, while
not changing much the theoretical bound, improves
drastically the practical convergence.
\end{document}